%% file: OwC_online.tex
\documentclass[11pt]{article}

\usepackage{mystyle}

\usepackage{xcolor}

\usepackage[utf8]{inputenc} 
\usepackage[T1]{fontenc}    
\usepackage{hyperref}       
\usepackage{url}            
\hypersetup{colorlinks,linkcolor={red},citecolor={blue},urlcolor={red}}  
\usepackage{booktabs}       
\usepackage{amsfonts}       
\usepackage{nicefrac}       
\usepackage{microtype}      %
\usepackage{lmodern}

\newtheorem*{lemma*}{Lemma}

\usepackage{algorithm}
\usepackage{algpseudocode}
\usepackage{setspace}

\usepackage{graphicx}				
\usepackage{amssymb}
\usepackage{amsmath}
\usepackage{amsthm}
\usepackage{appendix}
\usepackage{enumitem}

\def\l{\left}
\def\r{\right}

\newcommand{\mc}[1]{\mathcal{#1}}

\title{Upper Confidence Primal-Dual Reinforcement Learning for CMDP with Adversarial Loss}

\begin{document}
\author{Shuang Qiu\thanks{University of Michigan. 
Email: \texttt{qiush@umich.edu}.} 
       \qquad
	Xiaohan Wei\thanks{University of Southern California.
Email: \texttt{xiaohanw@usc.edu}.}  
	\qquad
       Zhuoran Yang\thanks{
       Princeton University. 
    Email: \texttt{zy6@princeton.edu}.}
	\qquad
       Jieping Ye\thanks{University of Michigan. 
    Email: \texttt{jpye@umich.edu}.}
	\qquad    
       Zhaoran Wang\thanks{Northwestern University.
    Email: \texttt{zhaoranwang@gmail.com}.}
       }       

\maketitle

\begin{abstract}
We consider online learning for episodic stochastically constrained Markov decision processes (CMDPs), which plays a central role in ensuring the safety of reinforcement learning. Here the loss function can vary arbitrarily across the episodes, and both the loss received and the budget consumption are revealed at the end of each episode. Previous works solve this problem under the restrictive assumption that the transition model of the Markov decision processes (MDPs) is known a priori and establish regret bounds that depend polynomially on the cardinalities of the state space $\mathcal{S}$ and the action space $\mathcal{A}$. In this work, we propose a new \emph{upper confidence primal-dual} algorithm, which only requires the trajectories sampled from the transition model. In particular, we prove that the proposed algorithm achieves $\widetilde{\mathcal{O}}(L|\mathcal{S}|\sqrt{|\mathcal{A}|T})$ upper bounds of both the regret and the constraint violation, where $L$ is the length of each episode. Our analysis incorporates a new high-probability drift analysis of Lagrange multiplier processes into the celebrated regret analysis of upper confidence reinforcement learning, which demonstrates the power of   ``optimism in the face of uncertainty'' in constrained online learning.
\end{abstract}

\section{Introduction}
Constrained Markov decision processes play an important role in control and planning. It aims at maximizing a reward or minimizing a penalty metric over the set of all available policies subject to constraints on other relevant metrics. The constraints aim at enforcing the fairness or safety of the policies so that over time the behaviors of the chosen policy are under control. For example, in an edge cloud serving network \citep{urgaonkar2015dynamic, wang2015dynamic}, one would like to minimize the average cost of serving the moving targets subject to a constraint on the average serving delay.  In an autonomous vehicle control problem \citep{le2019batch}, one might be interested in minimizing the driving time subject to certain fuel efficiency or driving safety constraints. 

Classical treatment of CMDPs dates back to \citet{fox1966markov, altman1999constrained} reformulating the problem into a linear program (LP) via stationary state-action occupancy measures. However, to formulate such an LP, one requires the full knowledge of the transition model, reward, and constraint functions, and also assumes them to be fixed. Leveraging the episodic structure of a class of MDPs, \citet{neely2012dynamic} develops online renewal optimization which potentially allows the loss and constraint functions to be stochastically varying and unknown, while still relying on the transition model to solve the subproblem within the episode. 

More recently, policy-search type algorithms have received much attention, attaining state-of-art performance in various tasks without knowledge of the transition model, e.g.  \citet{williams1992simple, baxter2000direct, konda2000actor, kakade2002natural, schulman2015trust, lillicrap2015continuous, schulman2017proximal, sutton2018reinforcement, fazel2018global, abbasi2019politex, abbasi2019exploration, bhandari2019global, cai2019provably, wang2019neural,liu2019neural,agarwal2019optimality}. While most of the algorithms focus on unconstrained reinforcement learning problems, there are efforts to develop policy-based methods in CMDPs where constraints are known with limited theoretical guarantees. The work \citet{chow2017risk} develops a primal-dual type algorithm which is shown to converge to some constraint satisfying policy. \citet{achiam2017constrained} develops a trust-region type algorithm which requires solving an optimization problem with both trust-region and safety constraints during each update. Generalizing ideas from the fitted-Q iteration, \citet{le2019batch} develops a batch offline primal-dual type algorithm which guarantees only the time average primal-dual gap converges. 

The goal of this paper is to solve constrained episodic MDPs with more generality where not only transition models are unknown, but also the loss and constraint functions can change online. In particular, the losses can be arbitrarily time-varying and adversarial. Let $T$ be the number of episodes\footnote{In the non-episodic (infinite-horizon) setting, $T$ denotes the total number of steps, which is a little different from the aforementioned $T$ for the episodic setting.}. When assuming the transition model is known, \citet{even2009online} achieves $\widetilde{\mathcal{O}}(\varrho^2 \sqrt{T \log |\mathcal{A}|})$ regret with $\varrho$ being the mixing time of the MDPs, and the work \citet{yu2009markov} achieves $\widetilde{\mathcal{O}}(T^{2/3})$ regret. These two papers consider a continuous setting that is a little different to the episodic setting that we consider in this paper. \citet{zimin2013online} further studies the episodic MDP and achieves $\widetilde{\mathcal{O}}(L \sqrt{ T \log (|\mathcal{S}||\mathcal{A}|)})$ regret. For the constrained case with known transitions, the work \citet{wei2018online} achieves an $\widetilde{\mathcal{O}}(\text{poly}(|\mc S||\mc A|)\sqrt{T})$ regret and constraint violations, and \citet{zheng2020constrained} attains $\widetilde{\mathcal{O}}(|\mc S||\mc A|T^{3/4})$ for the non-episodic setting.

After we finished the first version of this work, there are several concurrent works appearing which also focus on CMDPs with unknown transitions and losses. The work \citet{efroni2020exploration} 
studies episodic tabular MDPs with unknown but fixed loss and constraint functions, where the feedbacks are stochastic bandits. Leveraging upper confidence bound (UCB) on the reward, constraints, and transitions, they obtain an $\mathcal{O}(\sqrt{T})$ regret and constraint violation via linear program as well as primal-dual optimization. In another work, \citet{ding2020provably} studies the constrained episodic MDPs with a linear structure and adversarial losses via a primal-dual-type policy optimization algorithm, achieving $\widetilde{ \mathcal{O}} (\sqrt{T})$ regret and constraint violation. While the scenario in \citet{ding2020provably} is more general than ours, their dependence on the length of episode is worse when applied to the tabular case. Both of these two works rely on Slater condition which is also more restrictive than that of this work.

On the other hand, for unconstrained online MDPs, the idea of UCB has shown to be effective and helps to achieve tight regret bounds without knowing the transition model, e.g. \citet{jaksch2010near, azar2017minimax, rosenberg2019online, rosenberg2019onlineb, jin2019learning}. The main idea there is to sequentially refine a confidence set of the transition model and choose a model in the interval which performs the best in optimizing the current value.

The main contribution of this paper is to show that UCB is also effective when incorporating with primal-dual type approaches to achieve $\widetilde{\mathcal{O}}(L|\mc S|\sqrt{|\mc A|T})$ regret and constraint violation simultaneously in online CMDPs when the transition model is unknown, the loss is adversarial, and the constraints are stochastic. This almost matches the lower bound  $\Omega(\sqrt{ L|\mathcal{S}| |\mathcal{A}|T})$ for the regret \cite{jaksch2010near} up to an $\tilde{\mathcal{O}}(\sqrt{L|\mathcal{S}|})$ factor. Under the hood is a new Lagrange multiplier condition together with a new drift analysis on the Lagrange multipliers leading to low constraint violation. Our setup is challenging compared to classical constrained optimization in particular due to \textbf{(1)} the unknown loss and constraint functions from the online setup; \textbf{(2)} the time-varying decision sets resulting from moving confidence set estimation of UCB. 
The decision sets can potentially be much larger than or even inconsistent with the true decision set knowing the model, resulting in a potentially large constraint violation. The main idea is to utilize a Lagrange multiplier condition as well as a confidence set of the model to construct a probabilistic bound on an online dual multiplier. We then explicitly take into account the laziness nature of the confidence set estimation in our algorithm to argue that the bound on the dual multiplier gives the $\tilde{\cO}(\sqrt{T})$ bound on constraint violation.

\section{Related Work}
In this paper, we are more interested in a class of online MDP problems where the loss functions are arbitrarily changing, or adversarial. With a known transition model, adversarial losses, and full-information feedbacks (as opposed to bandit feedbacks), \citet{even2009online} achieves $\widetilde{\mathcal{O}}(\varrho^2 \sqrt{T \log |\mathcal{A}|})$ regret with $\varrho$ being the mixing time of the MDP, and the work \citet{yu2009markov} achieves $\widetilde{\mathcal{O}}(T^{2/3})$ regret. These two papers consider a non-episodic setting that is different to the episodic setting that we consider in this paper. The work \citet{zimin2013online} further studies the episodic MDP and achieves an $\widetilde{\mathcal{O}}(L \sqrt{ T \log (|\mathcal{S}||\mathcal{A}|)})$ regret.

A more challenging setting is that the transition model is unknown. Under such setting, there are several works studying the online episodic MDP problems with adversarial losses and full-information feedbacks. \citet{neu2012adversarial} obtains $\widetilde{\mathcal{O}}(L |\mc S| |\mc A| \sqrt{T})$ regret by proposing a Follow the Perturbed
Optimistic Policy (FPOP) algorithm. The recent work \citet{rosenberg2019online} improves the
regret to $\widetilde{\mathcal{O}}(L|\mathcal{S}|\sqrt{|\mathcal{A}|T})$ by proposing an online upper confidence mirror descent algorithm. This regret bound nearly matches the lower bound $\Omega(\sqrt{ L|\mathcal{S}| |\mathcal{A}|T})$ \citep{jaksch2010near} up to $\mathcal{O}(\sqrt{L|\mathcal{S}|})$ and some logarithm factors. Our work is along this line of research, and further considers the setup that there exist stochastic constraints observed at each episode during the learning process.

Besides, a number of papers also investigate online episodic MDPs with bandit feedbacks. Assuming the transition model is known and the losses are adversarial, \citet{neu2010online} achieves $\widetilde{\mathcal{O}}( L^2 \sqrt{T |\mathcal{A}|}/\beta)$ regret, where $\beta$ is the probability that all states are reachable under all policies. Under the same setting, \citet{neu2010online} achieves $\widetilde{\mathcal{O}}(T^{2/3})$ regret without the dependence on $\beta$, and \citet{zimin2013online} obtains $\widetilde{\mathcal{O}}(\sqrt{L|\mathcal{S}||\mathcal{A}|T})$ regret.  Furthermore, with assuming the transition model is not known and the losses are adversarial, \citet{rosenberg2019onlineb} obtains $\widetilde{\mathcal{O}}(T^{3/4})$ regret and also $\widetilde{\mathcal{O}}(
\sqrt{T} /\beta)$ where all states are reachable with probability $\beta$ under any policy. \citet{jin2019learning} further achieves $\widetilde{\mathcal{O}}(L|\mathcal{S}| \sqrt{|\mathcal{A}|T})$ regret under the same setting of the unknown transition model and adversarial losses. 
We remark that our algorithm can be extended to the setting of constrained episodic MDP where both the loss and constraint  functions are time-varying and we only receive bandit feedbacks. 
We leave such an extension as our future work. 

On the other hand, instead of adversarial losses, extensive works have studied the setting where the feedbacks of the losses are stochastic and have fixed expectations, e.g. \citet{jaksch2010near, azar2017minimax, ouyang2017learning, jin2018q, fruit2018efficient, wei2019model, zhang2019regret, dong2019q}. With assuming that the transition model is known, \citet{zheng2020constrained} studies online CMDPs under the non-episodic setting and attains an $\widetilde{\mathcal{O}}(|\mc S||\mc A|T^{3/4})$ regret which is suboptimal in terms of $T$. The concurrent work \citet{efroni2020exploration} 
studies episodic MDPs with unknown transitions and stochastic bandit feedbacks of the losses and the constraints, and obtains an $\tilde{\mathcal{O}}(\sqrt{T})$ regret and constraint violation.

In addition to the aforementioned papers, there is also a line of policy-search type works, focusing on solving online MDP problems via directly optimizing policies without knowing the transition model, e.g., \citet{williams1992simple, baxter2000direct, konda2000actor, kakade2002natural, schulman2015trust, lillicrap2015continuous, schulman2017proximal, sutton2018reinforcement, fazel2018global, abbasi2019politex, abbasi2019exploration, bhandari2019global, cai2019provably, wang2019neural,liu2019neural,agarwal2019optimality,efroni2020optimistic}. Efforts have also been made in several works \citep{chow2017risk, achiam2017constrained, le2019batch} to investigate CMDP problems via policy-based methods, but with known transition models. In another concurrent work, assuming the transition model is unknown, \cite{ding2020provably} studies the episodic CMDPs with linear structures and proposes a primal-dual type policy optimization algorithm.

\section{Problem Formulation}\label{sec:prob}

Consider an episodic loop-free MDP with a finite state space $S$ and a finite action space $\mathcal A$ at each state over a finite horizon of $T$ episodes. Each episode starts with a fixed initial state $s_0$ and ends with a terminal state $s_L$.
The transition probability is $P:\mc S\times \mc S\times \mc A\rightarrow [0,1]$, where $P(s'|s,a)$ gives the probability of transition from $s$ to $s'$ under an action $a$. This underlying transition model $P$ is assumed to be unknown. The state space is loop-free, i.e., it is divided into layers, i.e., $\mathcal S := \mc S_0\cup \mc S_1\cup\cdots\cup\mc S_L$ with a singleton initial layer 
$\mc S_0= \{s_0\}$ and terminal layer $S_L= \{s_L\}$. Furthermore, $\mc S_k\cap \mc S_\ell = \emptyset,~k\neq \ell$ and transitions are only allowed between consecutive layers, which is $P(s'|s,a)>0$ only if $s'\in \mc S_{k+1}$, $s\in\mc S_k$, and $a\in\mc A$, $\forall k\in\{0,1,2,\cdots, L-1\}$. Note that such an assumption enforces that each path from the initial state to the terminal state takes a fixed length $L$. This is not an excessively restrictive assumption as any loop-free MDP with bounded varying path lengths can be transformed into one with a fixed path length (see \citet{gyorgy2007line} for details). 

The loss function for each episode is $f^t:\mc S\times\mc A\times \mc S\rightarrow\mathbb R$, where $f^t(s,a,s')$ denotes the loss received at episode $t$ for any $s\in\mc S_k,~s'\in\mc S_{k+1}$ and $a\in\mc A$, $\forall k\in\{0,1,2,\ldots,L-1\}$.  We assume $f_t$ can be arbitrarily varying with potentially no fixed probability distribution. There are $I$ stochastic constraint (or budget consumption) functions: $g_{i}^t:\mc S\times\mc A\times \mc S\rightarrow\mathbb R,~\forall i\in\{1,2,\ldots,I\}$, where $g_i^t(s,a,s')$ denotes the price to pay at episode $t$ for any $(s,a,s')$. 
Each stochastic function $g_i^t$ at episode $t$ is sampled according to a random variable $\xi_i^t\sim\mathcal{D}_i$, namely $g_i^t(s,a,s') = g_i(s,a,s'; \xi_i^t)$. Then, we define $g_i(s,a,s') := \E[g_i^t(s,a,s')] = \E_{\xi_i^t}[g_i(s,a,s'; \xi_i^t)]$ where the expectation is taken over the randomness of $\xi_i^t\sim \mathcal{D}_i$. For abbreviation, we denote $g_i = \E[g_i^t]$.
In addition, the functions $f^t$ and $g_i^{t}, ~\forall i \in \{1,\ldots, I\}$, are mutually independent and independent of the Markov transitions. Both the loss functions and the budget consumption functions are revealed at the end of each episode. 

\begin{remark}
It might be tempting to consider the more general scenario that both losses and constraints are arbitrarily time-varying. For such a setting, however, there exist counterexamples \citep{mannor2009online} in the arguably simpler constrained online learning scenario that no algorithm can achieve sublinear regret and constraint violation simultaneously. Therefore, we seek to put extra assumptions on the problem so that obtaining sublinear regret and constraint violation is feasible, one of which is to assert constraints to be stochastic instead of arbitrarily varying.
\end{remark}

For any episode $t$, a policy $\pi_t$ is the conditional probability $\pi_t(a|s)$ of choosing an action $a\in\mathcal A$ at any given state $s\in\mc S$.
Let $(s_k, a_k, s_{k+1}) \in \mc S_k\times \mc A \times \mc S_{k+1}$ denotes a random tuple generated according to the transition model $P$ and the policy $\pi_t$. The corresponding expected loss is  $\sum_{k=0}^{L-1}\expect{f^t(s_k,a_k,s_{k+1})| \pi_t, P}$, while the budget costs are $\sum_{k=0}^{L-1}\expect{g_i^t(s_k,a_k,s_{k+1})| \pi_t, P}, i\in\{1,\cdots,I\}$, where 
the expectations are taken w.r.t. 
the randomness of the tuples $(s_k, a_k, s_{k+1})$. 
In this paper, we adopt the occupancy measure $\theta(s,a,s')$ for our analysis. In general, the occupancy measure $\theta(s,a,s')$ is a joint distribution of the tuple $(s,a,s') \in \mathcal{S} \times \mathcal{A} \times \mathcal{S}$ under some certain policy and transition model. 
Particularly, under the true transition $P$, we define the set $\Delta := \{ \theta ~:~  \theta \text{ satisfies \ref{item:cond1}, \ref{item:cond2}, and \ref{item:cond3}} \}$ \citep{altman1999constrained} where
\begin{enumerate}[label=(\alph*)] 
\item   \label{item:cond1} $\sum_{s\in \mc S_k} \sum_{a\in\mc A} \sum_{s'\in\mc S_{k+1}} \theta(s,a,s')=1, \forall k \in \{0, \ldots, L-1\}, \text{ and }~ \theta(s,a,s')\geq0$. 
\item \label{item:cond2}  $\sum_{s\in\mc S_k} \sum_{a\in \mc A} \theta(s, a, s') = \sum_{a\in \mc A} \sum_{s''\in\mc S_{k+2}}  \theta(s', a, s''), \forall k\in \{ 0, \ldots, L-2 \}, ~s' \in \mc S_{k+1}$.
\item  \label{item:cond3} $ \theta(s,a,s')/\sum_{s''\in \mc S_{k+1}}\theta(s,a,s'') = P(s'|s,a), \forall k\in \{ 0, \ldots, L-1 \}, s \in \mc S_k, a\in \mc A, s' \in\mc S_{k+1}$. 
\end{enumerate}
We can further recover a policy $\pi$ from $\theta$ via 
$\pi(a|s) = \sum_{s'\in\mc S_{k+1}} \theta(s,a,s')/\sum_{s'\in\mc S_{k+1},a\in\mc A} \theta(s,a,s')$.

We define $\overline{\theta}^t(s,a,s')$, $s\in\mc S,~s'\in\mc S,~a\in\mc A$, to be the occupancy measure at episode $t$ w.r.t. the true transition $P$, resulting from a policy $\pi_t$ at episode $t$. Given the definition of occupancy measure, we can rewrite the expected loss and the budget cost as $\sum_{k=0}^{L-1}\expect{f^t(s_k,a_k,s_{k+1})| \pi_t, P} = \langle f^t, \overline{\theta}^t \rangle$ where $\langle f^t, \overline{\theta}^t \rangle = \sum_{s,a,s'} f^t(s,a,s') \overline{\theta}^t(s,a,s')$ and $\sum_{k=0}^{L-1}\expect{g_i^t(s_k,a_k,s_{k+1})| \pi_t, P} = \langle g_i^t, \overline{\theta}^t \rangle$  with $\langle g_i^t, \overline{\theta}^t \rangle = \sum_{s,a,s'} f^t(s,a,s') \overline{\theta}^t(s,a,s')$. 
We aim to solve the following constrained optimization, and let $\overline{\theta}^*$ be one solution which is further viewed as a reference point to define the regret:
\begin{align}
\begin{aligned} \label{eq:prob-3}
\minimize_{\theta\in\Delta}~~&\sum_{t=1}^{T}  \langle f^t, \theta \rangle , ~~~\text{subject to}~~~
 \langle g_i, \theta \rangle  \leq c_i, ~~
\forall i\in[I],
\end{aligned}
\end{align}
where $\sum_{t=1}^{T}\langle f^t, \theta \rangle = \sum_{t=1}^{T}\E[ \sum_{k=0}^{L-1} f^t(s_k,a_k,s_{k+1})| \pi, P]$ is the overall loss in $T$ episodes and constraints are enforced on the budget cost $\langle g_i, \theta \rangle = \E [ \sum_{k=0}^{L-1} g_i(s_k,a_k,s_{k+1})| \pi, P]$ based on the expected budget consumption functions $g_i, \forall i \in [I]$. 
To measure the regret and the constraint violation respectively for solving \eqref{eq:prob-3} in an online setting, we define the following two metrics:
\begin{align}
\mathrm{Regret}(T):=\sum_{t=1}^{T} \big\langle f^t, \overline{\theta}^t - \overline{\theta}^* \big\rangle, ~~~\text{ and }~~ \mathrm{Violation}(T):=\Bigg\| \Bigg[ \sum_{t=1}^{T} ( \mf g^t(\overline{\theta}^t) - \mf c ) \Bigg]_{+} \Bigg\|_2,
\end{align}
where the notation $[\mf v]_+$ denotes the entry-wise application of $\max\{\cdot,0\}$ for any vector $\mf v$.  For abbreviation, let $\mf g^t(\theta) := [\dotp{g_1^t}{\theta},~\cdots,~\dotp{g_I^t}{\theta}]^\top$, 
and $\mf c := \l[ c_1,~\cdots,~c_I\r]^\top$. 

The goal is to attain a sublinear regret bound and constraint violation on this problem w.r.t. 
\textit{any fixed stationary policy $\pi$, which does not change over episodes}. In another word, we compare to the best policy $\pi^*$ in hindsight whose corresponding occupancy measure $\overline{\theta}^*\in \Delta$ solves problem \eqref{eq:prob-3}.
We make the following assumption on the existence of a solution to \eqref{eq:prob-3}.
\begin{assumption}\label{as:feasible}
There exists at least one fixed policy $\pi$ such that the corresponding probability $\theta \in \Delta$ is feasible, i.e.,
$\langle g_i, \theta \rangle  \leq c_i,
i\in\{1,2,\cdots,I\}$.
\end{assumption}
Then, we assume the following boundedness on function values for simplicity of notations without loss of generality.
\begin{assumption}\label{as:bounded}
We assume the following quantities are bounded. For any $t\in\{1,2,\ldots,T\}$,  \textbf{(1)} $\sup_{s,a,s'} |f^t(s,a,s')|\leq 1$, \textbf{(2)} $\sum_{i=1}^I\sup_{s,a,s'}|g_i^t(s,a,s')|\leq 1$, \textbf{(3)} $\sum_{i=1}^I |c_i|\leq L$. 
\end{assumption}

%

When the transition model $P$ is known and Slater's condition holds (i.e., existence of a policy which satisfies all stochastic inequality constraints with a constant $\varepsilon$-slackness), this stochastically constrained online linear program can be solved via similar methods as \citet{wei2018online, yu2017online} with a regret bound that depends polynomially on the cardinalities of state and action spaces, which is highly suboptimal especially when the state or action space is large. The main challenge we will address in this paper is to \textit{solve this problem without knowing the model $P$, or losses and constraints before making decisions, while tightening the dependency on both state and action spaces in the resulting performance bound.}

\section{Proposed Algorithm}

In this section, we introduce our proposed algorithm, namely, the upper confidence primal-dual (UCPD) algorithm, as presented in Algorithm \ref{alg:ucpd}. It adopts a primal-dual mirror descent type algorithm solving constrained problems but with an important difference: 
We maintain a confidence set via past sample trajectories, which contains the true MDP model $P$ with high probability, and choose the policy to minimize the proximal Lagrangian using the most optimistic model from the confidence set. Such an idea, known as \textit{optimism in the face of uncertainty}, is reminiscent of the upper confidence bound (UCB) algorithm \citep{auer2002finite} for stochastic multi-armed bandit (MAB) and first proposed by \citet{jaksch2010near} to obtain a near-optimal regret for reinforcement learning problems.

In the algorithm, we introduce epochs, which are back-to-back time intervals that span several episodes. We use $\ell\in\{1,2,\cdots\}$ to index the epochs and use $\ell(t)$ to denote a mapping from episode index $t$ to epoch index, indicating which epoch the $t$-th episode lives in. 
Next, let $N_\ell(s,a)$ and $M_\ell(s,a,s')$ be two global counters which indicate the number of times the tuples $(s,a)$ and $(s,a,s')$ appear before the $\ell$-th epoch. Let 
$n_\ell(s,a)$, $m_\ell(s,a,s')$ be two local counters which indicate the number of times the tuples $(s,a)$ and $(s,a,s')$ appear in the $\ell$-th epoch. 
We start a new epoch whenever there exists $(s,a)$ such that $n_{\ell(t)}(s,a)\geq N_{\ell(t)}(s,a)$. Otherwise, set $\ell(t+1) = \ell(t)$. Such an update rule follows from \citet{jaksch2010near}. Then, we define the empirical transition model $\widehat P_\ell$ at any epoch $\ell>0$ as 
$$
\widehat P_\ell(s'|s,a) := \frac{M_{\ell}(s,a,s')}{\max\{1,N_{\ell}(s,a)\}},~~\forall s,s'\in\mc S,~a\in\mc A.
$$
As shown in Remark \ref{re:change}, introducing the notion of `epoch' is necessary to achieve an $\widetilde{\mathcal{O}} (\sqrt{T})$ constraint violation.

The next lemma shows that with high probability, the true transition model $P$ is contained in a confidence interval around the empirical one no matter what sequence of policies taken, which is adapted from Lemma 1 of \citet{neu2012adversarial}.
\begin{lemma}[Lemma 1 of \citet{neu2012adversarial}]\label{lem:confidence}
For any $\zeta\in(0,1)$, we have that with probability at least $1-\zeta$, for all epoch $\ell\leq \ell(T+1)$ and any state and action pair $(s,a)\in\mc S\times\mc A$,
$$
\Big\|P(\cdot|s,a) -\widehat P_\ell(\cdot|s,a) \Big \|_1\leq \varepsilon_\ell^{\zeta}(s,a),
$$
with the error $\varepsilon_\ell^{\zeta}(s,a)$ being
\begin{equation}\label{eq:confidence-interval}
\varepsilon_\ell^{\zeta}(s,a):= \sqrt{\frac{2|\mc S_{k(s)+1}|\log[(T+1)|\mc S||\mc A|/\zeta]}{\max\{1,N_{\ell}(s,a)\}}},
\end{equation}
where $k(s)$ is a mapping from state $s$ to the layer that the state $s$ belongs to. 
\end{lemma}

\begin{algorithm}[!t]\caption{Upper-Confidence Primal-Dual (UCPD) Mirror Descent} 
	\begin{algorithmic}[1] 
		\State {\bfseries Input:} Let $V, \alpha> 0$,~$\lambda\in[0,1)$ be some trade-off parameters. Fix $\zeta\in(0,1)$.
		\State {\bfseries Initialize:} $Q_i(1) = 0,~\forall i=1,\ldots, I$. $\theta^1(s,a,s') = 1/(|\mc S_k||\mc S_{k+1}||\mc A|), ~\forall (s,a,s')\in \mc S_k\times\mc A\times\mc S_{k+1}$. $\ell(1) = 1$. $n_{1}(s,a) = 0,~N_{1}(s,a) = 0,~\forall (s,a)\in\mc S\times\mc A$. $m_{1}(s,a,s') = 0,~M_{1}(s,a,s') = 0,~\forall (s,a,s') \in \mc S \times \mc A \times \mc S$.
		\For{$t=1, 2, 3, \ldots$}
			\State Compute $\theta^t$ via \eqref{eq:optimistic-lp} and the corresponding policy $\pi_t$ via \eqref{eq:retrieve-policy}.
			\State Sample a path $(s_0^t, a_0^t,\cdots,s_{L-1}^t,a_{L-1}^t,s_L^t)$ following the policy $\pi_t$. 
			\State Update each dual multiplier $Q_i(t)$ via \eqref{eq:Q-update} and update the local counters:
\begin{align*}
&n_{\ell(t)}(s_k^{t},a_k^{t}) = n_{\ell(t)}(s_k^{t},a_k^{t}) + 1, \\
&m_{\ell(t)}(s_k^{t},a_k^{t},s_{k+1}^{t}) = m_{\ell(t)}(s_k^{t},a_k^{t},s_{k+1}^{t}) + 1.
\end{align*} 
			\State Observe the loss function $f^{t}$ and constraint functions $\{g_i^t\}_{i=1}^I$. 
			\If{$\exists (s,a)\in\mc S\times\mc A, ~n_{\ell(t)}(s,a)\geq N_{\ell(t)}(s,a),$}
				\State \textbf{Start a new epoch:}
				\State Set $\ell(t+1) = \ell(t)+1$, and update the global counters for all $s,s'\in\mc S,~a\in\mc A$ by
					\begin{align*}
						&N_{\ell(t+1)}(s,a) = N_{\ell(t)}(s,a) + n_{\ell(t)}(s,a),\\
						&M_{\ell(t+1)}(s,a,s') = M_{\ell(t)}(s,a,s') + m_{\ell(t)}(s,a,s').
					\end{align*}
				\State Construct the empirical transition 
				\begin{align*}
				\widehat P_{\ell(t+1)}(s'|s,a) := \frac{M_{\ell(t+1)}(s,a,s')}{\max\{1,N_{\ell(t+1)}(s,a)\}}, \forall(s,a, s').
				\end{align*}

				\State Initialize $n_{\ell(t+1)}(s,a) = 0,~m_{\ell(t+1)}(s,a,s') = 0,~\forall (s,a,s')\in\mc S\times\mc A\times \mc S$.
			\Else 
				\State Set $\ell(t+1)=\ell(t)$.
			\EndIf
        \EndFor
	\end{algorithmic}\label{alg:ucpd}
\end{algorithm}

\subsection{Computing Optimistic Policies}

Next, we show how to compute the policy at each episode. Formally, we introduce a new occupancy measure at episode $t$, namely $\theta^t(s,a,s'),~s,s'\in\mc S,~a\in\mc A$.  It should be emphasized that this is different from the $\overline\theta^t(s,a,s')$ defined in the previous section as $\theta^t(s,a,s')$ is chosen by the decision maker at episode $t$ to construct the policy.
In particular, $\theta^t(s,a,s')$ does not have to satisfy the local balance equation \ref{item:cond3}. 
Once getting $\theta^t(s,a,s')$ (which will be detailed below), we construct the policy as follows
\begin{equation}\label{eq:retrieve-policy}
\pi_t(a|s)= \frac{\sum_{s'}\theta^t(s,a,s')}{\sum_{s',a}\theta^t(s,a,s')},~\forall a\in\mc A,~s\in\mc S.
\end{equation}
Next, we demonstrate the proposed method computing $\theta^t(s,a,s')$. First, we introduce an online dual multiplier $Q_i(t)$ for each constraint in \eqref{eq:prob-3}, which is $0$ when $t=1$ and is updated as follows for $t \geq 2$,
\begin{equation}\label{eq:Q-update}
Q_i(t) = \max\{Q_i(t-1) + \dotp{g_i^{t-1}}{\theta^t} - c_i,~0\}.
\end{equation}
At each episode, we compute the new occupancy measure $\theta^t(s,a,s')$ solving an optimistic regularized linear program (ORLP) with tuning parameters $\lambda,~V,~\alpha>0$. Specifically, we update $\theta^t$ for all $t\geq 2$ by solving
\begin{align}
\begin{aligned}\label{eq:optimistic-lp}
\theta^{t}=\argmin_{\theta\in\Delta(\ell(t), \zeta)} ~~&\Big\langle Vf^{t-1} + \sum_{i=1}^IQ_i(t-1)g_i^{t-1}, \theta \Big\rangle + \alpha D(\theta, \widetilde{\theta}^{t-1}).
\end{aligned}
\end{align}
For $t=1$, we let $\theta^1(s,a,s') = 1/(|\mathcal{S}_k||\mathcal{S}_{k+1}||\mathcal{A}|)$, $\forall (s,a,s')\in \mathcal{S}_k\times \mathcal{A} \times \mathcal{S}_{k+1}$.
The above updating rule introduces extra notations $\Delta(\ell(t), \zeta)$, $\widetilde{\theta}^{t-1}$, and $D(\cdot, \cdot)$ that will be elaborated below.  Specifically,  we denote by $D(\cdot, \cdot)$ the unnormalized Kullback-Leibler(KL) divergence for two different occupancy measures $\theta$ and $\theta'$, which is defined as 
\begin{align}\label{eq:un-KL}
D(\theta,\theta'):= \sum_{s,a,s'}[\theta(s,a,s')\log \frac{\theta(s,a,s')}{\theta'(s,a,s')} - \theta(s,a,s') + \theta'(s,a,s')], \quad \forall \theta, \theta'.
\end{align}
In addition, for $ \forall k = \{0,\ldots, L-1\}$ and $ \forall s\in \mc S_k, a\in \mc A, s' \in \mc S_{k+1}$, we compute $\widetilde{\theta}^{t-1}$ via 
$$
\widetilde{\theta}^{t-1}(s,a,s') =(1-\lambda) \theta^{t-1}(s,a,s') + \frac{\lambda}{ |\mc S_k | |\mc S_{k+1} | |\mc A|},
$$
where $0 \leq \lambda \leq 1$. This equation introduces a probability mixing, pushing the update away from the boundary and encouraging explorations. 

Furthermore, since for any epoch $\ell > 0$, we can compute the empirical transition model $\widehat{P}_{\ell}$ with the confidence interval size $\varepsilon_{\ell}^\zeta$ as defined in \eqref{eq:confidence-interval}, we let every $\theta\in \Delta(\ell,\zeta)$ satisfy that 
\begin{align}
\bigg\|\frac{\theta(s,a,\cdot)}{\sum_{s'}\theta(s,a,s')}  - \widehat P_\ell(\cdot|s,a)\bigg\|_1\leq 
\varepsilon_\ell^{\zeta}(s,a), ~~\forall s \in \mc S, a \in \mc A, \label{eq:approx_local_balance}
\end{align}
such that we can define the feasible set $\Delta(\ell,\zeta)$ for the optimization problem \eqref{eq:optimistic-lp} as follows
\begin{align}\label{eq:subpro_fea_set}
\Delta(\ell,\zeta):= \{ \theta: \theta \text{ satisfies } \ref{item:cond1}, \ref{item:cond2}, \text{ and } \eqref{eq:approx_local_balance} \}.    
\end{align}
By this definition, we know that $\theta^t \in \Delta(\ell(t), \zeta)$ at the epoch $\ell(t)$. On the other hand, according to Lemma \ref{lem:confidence}, we have that with probability at least $1-\zeta$, for all epoch $\ell$, $\Delta \subseteq \Delta(\ell, \zeta)$ holds.
By 
 \citet{rosenberg2019online}, the problem \eqref{eq:optimistic-lp} is essentially a linear programming with a special structure that can be solved efficiently (see details in Section \ref{sec:linear_solver} of the supplementary material).

\section{Main Results} \label{sec:main}

Before presenting our results, we first make assumption on the existence of Lagrange multipliers. We define a partial average function starting from any time slot $t$ as $f^{(t,\tau)}: = \frac{1}{\tau}\sum_{j=0}^{\tau-1}f^{t+j}$. Then, we consider the following static optimization problem (recalling $g_i:=\E[g_i^t]$)
\begin{equation}\label{eq:partial-static-prob}
\minimize_{\theta\in\Delta} ~\langle f^{(t,\tau)}, \theta\rangle~~~\text{s.t.}~~~\dotp{g_i}{\theta}\leq c_i,~\forall i \in \{1, \ldots, I\}.
\end{equation}
Denote the solution to this program as $\theta^*_{t,\tau}$. 
Define the Lagrangian dual function of \eqref{eq:partial-static-prob} as $q^{(t,\tau)}(\eta):= \min_{\theta\in\Delta} ~f^{(t,\tau)}(\theta) + \sum_{i=1}^I\eta_i(g_i(\theta)-c_i)$, where $\eta = [\eta_1, \ldots, \eta_I]^\top \in\mathbb{R}^I$ is a dual variable. We are ready to state our assumption:
\begin{assumption}\label{as:selm}
For any time slot $t$ and any time period $\tau$, the set of primal optimal solution to \eqref{eq:partial-static-prob} is non-empty. Furthermore, the set of Lagrange multipliers, which is 
$\mathcal{V}_{t,\tau}^* := \text{argmax}_{\eta\in\mathbb{R}^I_+}q^{(t,\tau)}(\eta)$, is non-empty and bounded. Any vector in $\mathcal{V}^*_{t,\tau}$ is called a Lagrange multiplier associated with \eqref{eq:partial-static-prob}.
Furthermore, let $B>0$ be a constant such that for any $t\in\{1,\ldots,T\}$ and $\tau=\sqrt{T}$, the dual optimal set 
$\mathcal{V}_{t,\tau}^*$ defined above satisfies $\max_{\eta\in\mathcal{V}_{t,\tau}^*}\|\eta\|_2\leq B$.
\end{assumption}

As is discussed in Section \ref{sec:A0} of the supplementary material, Assumption \ref{as:selm} proposes a weaker condition than the Slater condition commonly adopted in previous constrained online learning works.  The following lemma further shows the relation between Assumption \ref{as:selm} and the dual function.
\begin{lemma}\label{lem:leb}
Suppose Assumption \ref{as:selm} holds, then for any $t\in\{1,\ldots,T\}$ and $\tau=\sqrt{T}$, there exists constants $\vartheta,~\sigma>0$ such that for any $\eta\in\R^I$ satisfying \footnote{We let $\dist( \eta, \mathcal{V}_{t,\tau}^*) := \argmin_{\eta' \in \mathcal{V}_{t,\tau}^*} \|\eta - \eta' \|_2$ as Euclidean distance between a point $\eta$ and the set $\mathcal{V}_{t,\tau}^*$.}  $\dist(\eta, \mathcal{V}_{t,\tau}^*)\geq \vartheta$, we have 
\begin{equation*}
q^{(t,\tau)}(\eta^*_{t,\tau })-q^{(t,\tau)}(\eta) \geq \sigma \cdot \dist(\eta, \mathcal{V}_{t,\tau}^*), ~~\forall ~ \eta^*_{t,\tau } \in \mathcal{V}_{t,\tau}^*.
\end{equation*}
\end{lemma}

Based on the above assumptions and lemmas, we present results of the regret and constraint violation. 
\begin{theorem}\label{thm:main}
Consider any fixed horizon $T\geq|\mc S||\mc A|$ with $ |\mc S |, |\mc A| > 1$. Suppose Assumption \ref{as:feasible}, \ref{as:bounded}, \ref{as:selm} hold and there exist absolute constants $\overline{\sigma}$ and $\overline \vartheta$ such that 
$\sigma\geq\overline{\sigma}$ and $\vartheta\leq\overline{\vartheta}$ for all $\sigma,~\vartheta$ in Lemma \ref{lem:leb} over $t\in\{1,\ldots, T\}$ and $\tau=\sqrt{T}$.
If setting $\alpha = LT,~V=L\sqrt{T}$, $\lambda=1/T$ and $\zeta \in  (0,1/(4+8L/\overline{\sigma})]$ in Algorithm \ref{alg:ucpd}, with probability at least $1-4\zeta$, we have
\begin{align*}
\mathrm{Regret}(T) &\leq  \widetilde{\mc O} \Big( L|\mathcal S| \sqrt{T|\mathcal A|} \Big), \\ \mathrm{Violation}(T) &\leq  \widetilde{\mc O} \Big( L|\mathcal S| \sqrt{ T|\mathcal A|} \Big),
\end{align*}
where the notation $\widetilde{O}(\cdot)$ absorbs the factors $\log^{3/2}(T/\zeta)$ and $ \log (T|\mathcal S| |\mathcal A|/\zeta)$.
\end{theorem}


\section{Theoretical Analysis} 
In this section, we provide lemmas and proof sketches for the regret bound and constraint violation bound in Theorem \ref{thm:main}. The detailed proofs for Section \ref{sec:regret} and Section \ref{sec:constraint} are in Section \ref{sec:A} and Section \ref{sec:B} of the supplementary material respectively. 

\subsection{Proof of Regret Bound} \label{sec:regret}

\begin{lemma}\label{lem:term_I} The updating rules in Algorithm \ref{alg:ucpd} ensure that with probability at least $1-2\zeta$,  
\begin{align*}
\sum_{t=1}^{T} \big\|\theta^t - \overline{\theta}^t \big\|_1 &\leq
(\sqrt 2+1)L|\mc S|\sqrt{2T|\mc A|\log \frac{2T|\mc S||\mc A|}{\zeta}} + 2L^2\sqrt{2T\log\frac{L}{\zeta}}. 
\end{align*} 
\end{lemma}

\begin{lemma} \label{lem:regret_primal_bound} The updating rules in Algorithm \ref{alg:ucpd} ensure that with probability at least $1-\zeta$,
\begin{align*}
\sum_{t=1}^{T} \big\langle f^t,\theta^t - \overline{\theta}^*\big \rangle  \leq& \frac{ 4 L^2 T  + (\lambda T+1)\alpha  L \log |\mathcal{S}|^2 |\mathcal{A}|}{V}    + 2\lambda LT  + \frac{LT}{2\alpha}+ \frac{1}{V}\sum_{t=1}^{T} \l\langle \mathbf{Q}(t), \mathbf{g}^t(\overline{\theta}^*)- \mathbf{c} \r\rangle  . 
\end{align*}
\end{lemma}

Here we let $\mathbf{Q}(t) := [Q_1(t),~Q_2(t),~\cdots,~Q_I(t)]^\top$. Next, we present Lemma \ref{lem:Q_drift_diff}, which is one of the key lemmas in our proof.  Then, this lemma indicates that $\|\mathbf{Q}(t)\|_2$ is bounded by $\mc O(\sqrt{T})$ with high probability when setting the parameters $\tau, V, \alpha, \lambda$ as in Theorem \ref{thm:main}. Thus, introducing stochastic constraints retains the $\mc O(\sqrt{T})$ regret. Moreover, this lemma will lead to the constraint violation in the level of $\mc O(\sqrt{T})$. Lemma \ref{lem:Q_drift_diff} is proved by using Assumption \ref{as:selm} and Lemma \ref{lem:leb}.

\begin{lemma}\label{lem:Q_drift_diff}
Letting $\tau=\sqrt{T}$ and $\zeta$ satisfy $\overline{\sigma}/4 \geq \zeta  (\overline{\sigma}/2 + 2 L)$, the updating rules in Algorithm \ref{alg:ucpd} ensure that with probability at least $1-T\delta$, the following inequality holds for all $t \in \{1,\ldots,T+1\}$, 
\begin{align*}
\|\mathbf{Q}(t)\|_2 \leq \omega : =\psi + \tau \frac{512L^2 }{ \overline{\sigma}} \log \bigg(1+\frac{128L^2}{\overline{\sigma}^2} e^{\overline{\sigma} / (32 L )} \bigg) + \tau \frac{64L^2}{\overline{\sigma}} \log \frac{1}{\delta}+2\tau L,
\end{align*}
where we define $\psi:=(2 \tau L + C_{V,\alpha, \lambda})/\overline{\sigma} + 2\alpha  L \log ( |\mathcal{S}|^2 |\mathcal A|/ \lambda) / (\overline{\sigma}\tau) + \tau \overline{\sigma}/2$ and
$C_{V, \alpha, \lambda} :=  2(\overline{\sigma} B+  \overline{\sigma} ~ \overline{\vartheta})V + (6  + 4\overline{\vartheta})  V L    +  VL/\alpha + 4L\lambda  V  + 2\alpha \lambda L \log |\mathcal S|^2 |\mathcal A|+ 8 L^2$.
\end{lemma}
The upper bound of $||Q(t)||_2$ in the above lemma actually holds for any $\tau$ and it is a convex function w.r.t. $\tau$, which thus indicates that there exists a tight upper bound of $||Q(t)||_2$ if $\tau$ is chosen by finding the minimizer of this upper bound. In this paper, we directly set $\tau=\sqrt{T}$, which suffices to give an $\tilde{\cO}(\sqrt{T})$ upper bound.

\begin{remark}\label{re:understand_term} We discuss the upper bound of the term $\log\big(1+\frac{128 L^2}{\overline{\sigma}^2} e^{\overline{\sigma} / (32 L )}\big)$ in the following way: \textbf{(1)} if $\frac{128 L^2 }{\overline{\sigma}^2} e^{\overline{\sigma} / (32L )} \geq 1$, then this term is bounded by $\log \big(\frac{256L^2}{\overline{\sigma}^2} e^{\overline{\sigma} / (32 L )} \big) =  \frac{\overline{\sigma}}{32L} + \log\frac{256L^2 }{\overline{\sigma}^2}$; \textbf{(2)} if $\frac{128 L^2}{\overline{\sigma}^2} e^{\overline{\sigma} / (32L)} < 1$, then the term is bounded by $\log 2$. Thus, combining the two cases, we have 
$$
\log \l(1+\frac{128L^2 }{\overline{\sigma}^2} e^{\overline{\sigma} / (32L)} \r) \leq \log 2 + \frac{\overline{\sigma}}{32 L} + \log\frac{256L^2 }{\overline{\sigma}^2}.
$$ 
This discussion shows that the $\log$ term in the result of Lemma \ref{lem:Q_drift_diff} will not introduce extra dependency on $L$ except a $\log L$ term.
\end{remark}

With the bound of $\|\mathbf{Q}(t)\|_2$ in Lemma \ref{lem:Q_drift_diff}, we further obtain the following lemma.

\begin{lemma}\label{lem:term_II} By Algorithm \ref{alg:ucpd}, if $\overline{\sigma}/4 \geq \zeta  (\overline{\sigma}/2 +2  L)$, then with probability at least $1 - 2T\delta$, 
\begin{align*}
&\sum_{t=1}^{T}\langle \mathbf{Q}(t), \mathbf{g}^t(\overline{\theta}^*)- \mathbf{c} \rangle \leq 2 L \omega \sqrt{T\log \frac{1}{T\delta}},
\end{align*}
with $\omega$ defined as the same as in Lemma \ref{lem:Q_drift_diff}.
\end{lemma}

\begin{proof}[Proof of the Regret Bound in Theorem \ref{thm:main}]
Recall that $\theta^t$ is the probability vector chosen by the decision maker, and $\overline{\theta}^t$ is the true occupancy measure at time $t$ while $\overline{\theta}^*$ is the solution to the problem \eqref{eq:prob-3}. 
The main idea is to decompose the regret as follows:
\begin{align}
\begin{aligned} \label{eq:regret_decomp}
\hspace*{-0.5cm}\sum_{t=1}^{T}\big\langle f^t, \overline{\theta}^t - \overline{\theta}^*\big \rangle &= \sum_{t=1}^{T} \Big( \big \langle f^t, \overline{\theta}^t - \theta^t \big \rangle +
\big \langle f^t , \theta^t - \overline\theta^*\big \rangle \Big) \\
& \leq   \underbrace{\sum_{t=1}^{T} \big\|\overline{\theta}^t - \theta^t \big\|_1}_{\text{Term (I)}} + \underbrace{\sum_{t=1}^{T} \big \langle f^t, \theta^t - \overline\theta^*\big \rangle}_{\text{Term (II)}} , 
\end{aligned}
\end{align}
where we use Assumption \ref{as:bounded} such that $\big\langle f^t, \overline{\theta}^t - \theta^t \big \rangle \leq \|f^t\|_\infty\big\|\overline{\theta}^t - \theta^t\big\|_1\leq \big\|\overline{\theta}^t - \theta^t\big\|_1$. Thus, it suffices to bound the Term (I) and Term (II).

We first show the bound for Term (I). According to Lemma \ref{lem:term_I}, by the fact that $L \leq |\mc S| $ and $|\mc S|, |\mc A| \geq 1$, we have that with probability at least $1-2\zeta$, the following holds
\begin{align} \label{eq:term_I_1}
\text{Term (I)}  \leq \mc O \l(L|\mc S|\sqrt{ T|\mc A|} \log^{\frac{1}{2}} ( T|\mc S||\mc A|/\zeta) \r). 
\end{align}
For Term (II), setting $V=L\sqrt{T}$, $\alpha = LT$, $\tau = \sqrt{T}$, and $\lambda = 1/T$, by Lemma \ref{lem:regret_primal_bound}, we obtain 
\begin{align*}
\text{Term (II)} \leq 8 L \sqrt{T |\mathcal{S}| |\mathcal{A}| } + \frac{1}{L\sqrt{T}}\sum_{t=1}^{T} \l\langle \mathbf{Q}(t), \mathbf{g}^t(\overline{\theta}^*)- \mathbf{c} \r\rangle , 
\end{align*}
where we use the inequality that $\log |\mc S| |\mc A| \leq \sqrt{|\mc S| |\mc A|} $ due to $\sqrt{x} \geq \log x$. Thus, we further need to bound the last term of the above inequality. By Lemma \ref{lem:term_II} and Remark \ref{re:understand_term}, with probability at least $1- 2T \delta$ for all  $t \in \{ 1,\ldots, T\}$, we have
$$
\frac{1}{L\sqrt{T}} \sum_{t=1}^{T}\l\langle \mathbf{Q}(t), \mathbf{g}^t(\overline{\theta}^*)- \mathbf{c} \r\rangle \leq  \mc O \l(L|\mc S|\sqrt{ T|\mc A|} \log^{\frac{3}{2}} (T/\delta) \r),
$$
by the facts that $L \leq |\mc S|$ , $|\mc S| > 1$, $|\mc A| > 1$, and the assumption $T \geq |\mc S| |\mc A|$, as well as the computation of $\psi$ as 
$$
\psi = \mc O \l( L^2\sqrt{T} + L \log |\mc S| |\mc A| + L^2 \sqrt{T} \log (T|\mc S| |\mc A|) \r).
$$
Therefore, with probability at least $1- 2 T\delta$, the following holds
\begin{align} \label{eq:term_I_2}
\text{Term (II)} \leq  \mc O \l(L|\mc S|\sqrt{ T|\mc A|} \log^{\frac{3}{2}} (T/\delta) \r).
\end{align}
Combining \eqref{eq:term_I_1} and \eqref{eq:term_I_2} with \eqref{eq:regret_decomp}, and letting $\delta = \zeta / T$, by union bound, we eventually obtain that with probability at least $1 -  4 \zeta$, the regret bound is
$$
\mathrm{Regret}(T) \leq \widetilde{\mc O} \l( L|\mc S|\sqrt{ T|\mc A|} \r),
$$ 
where the notation $\widetilde{\mc O}(\cdot)$ absorbs the logarithm factors. Further let $ \zeta \leq 1/(4+8L/\overline{\sigma}) < 1/4$ (such that $\overline{\sigma}/4 \geq \zeta  (\overline{\sigma}/2 + 2L )$ is guaranteed). This completes the proof.
\end{proof}

\subsection{Proof of Constraint Violation Bound}\label{sec:constraint}

\begin{lemma}\label{lem:term_III_all}
The updating rules in Algorithm \ref{alg:ucpd} ensure
\begin{align*}
\Bigg\| \Bigg[ \sum_{t=1}^{T} \Big( \mf g^t(\theta^t) - \mf c \Big) \Bigg]_{+} \Bigg\|_2  \leq \|\mathbf{Q}(T+1)\|_2 + \sum_{t=1}^{T}    \big\|\theta^{+1}  - \theta^{t} \big\|_1. 
\end{align*}
\end{lemma}

\begin{lemma}\label{lem:term_III_th_diff}
The updating rules in Algorithm \ref{alg:ucpd} ensure
\begin{align*}
\sum_{t=1}^T \big\|\theta^{t+1}-\theta^{t}\big\|_1 &\leq 3L  \sqrt{T|\mathcal S | |\mathcal A |} \log  \frac{8T}{ |\mathcal S | |\mathcal A | }  +  \frac{2L }{(1-\lambda)^2 \alpha } \sum_{t=1}^{T} \|\mathbf{Q}(t)\|_2 +\frac{2V LT }{(1-\lambda)^2 \alpha } \\
&\quad + \frac{2\lambda LT}{1-\lambda} +  \frac{\sqrt{8\lambda \log |\mathcal{S}|^2 |\mathcal{A}|}}{1-\lambda}LT.
\end{align*}
\end{lemma}
\begin{remark} \label{re:change}
The proof of Lemma \ref{lem:term_III_th_diff} uses the fact that the confidence set of the transition model $P$ changes only  $\sqrt{T|\mathcal S | |\mathcal A |}\log(8T/(|\mathcal S | |\mathcal A |))$ times due to the doubling of epoch length in Algorithm \ref{alg:ucpd}. Within each epoch where the confidence set is unchanged,  we further show $\|\theta^{t+1}-\theta^{t}\|_1$ is sufficiently small. Thus, the epoch length doubling eventually ensures that the constraint violation is in the level of $\tilde{\cO}(\sqrt{T})$. 
\end{remark}

\begin{proof}[Proof of the Constraint Violation Bound in Theorem \ref{thm:main}] We decompose $\mathrm{Violation}(T)$ as
\begin{align}
\begin{aligned}\label{eq:constr_decomp}
\Bigg\|\Bigg[\sum_{t=1}^{T} \Big(\mf g^t(\overline{\theta}^t) - \mf c\Big)\Bigg]_+\Bigg\|_2 & \leq \sum_{t=1}^{T} \big\|\mf g^t(\theta^t) - \mf g^t(\overline{\theta}^t)\big\|_2+ 
\Bigg\|\Bigg[\sum_{t=1}^{T}\l(\mf g^t(\theta^t) - \mf c\r)\Bigg]_+\Bigg\|_2  \\
&\leq \underbrace{\sum_{t=1}^{T}\big\|\overline{\theta}^t - \theta^t \big\|_1}_{\text{Term (III)}} +\underbrace{\Bigg\|\Bigg[\sum_{t=1}^{T}\Big(\mf g^t(\theta^t) - \mf c\Big)\Bigg]_+\Bigg\|_2}_{\text{Term (IV)}}, 
\end{aligned}
\end{align}
where the second inequality is due to Assumption \ref{as:bounded} that $\|\mf g^t(\theta^t) - \mf g^t(\overline{\theta}^t)\|_2= \big(\sum_{i=1}^I | \langle g_i^t, \theta^t - \overline{\theta}^t \rangle |^2\big)^{\frac{1}{2}} \leq \sum_{i=1}^I \|g_i^t\|_\infty\|\theta^t - \overline{\theta}^t\|_1\leq \|\theta^t - \overline{\theta}^t\|_1$. Thus, it suffices to bound Terms (III) and (IV).

For Term (III), we already have its bound as \eqref{eq:term_I_1}. 
Then, we focus on proving the upper bound of Term (IV). Set $V=L\sqrt{T}$, $\alpha = LT$, $\tau = \sqrt{T}$, and $\lambda = 1/T$ as in the proof of the regret bound. By Lemma \ref{lem:term_III_all}, we know that to bound Term (IV) requires bounding the terms $\|\mathbf{Q}(T+1)\|_2$ and $\sum_{t=1}^{T} \|\theta^{t+1}  - \theta^{t}\|_1$. By Lemma \ref{lem:Q_drift_diff}, combining it with Remark \ref{re:understand_term} and $\psi = \mc O \big( L^2\sqrt{T} + L \log |\mc S| |\mc A| + L^2 \log (T|\mc S| |\mc A|) /\sqrt{T} \big)$ as shown in the proof of the regret bound, letting $\overline{\sigma}/4 \geq \zeta  (\overline{\sigma}/2 + 2 L)$, with probability $1-T\delta$, for all $t \in \{1,\ldots, T+1\}$, the following inequality holds
\begin{align}
\|\mathbf{Q}(t)\|_2 \leq \mc O \l( L^2\sqrt{T} \log(L/\delta) \r), \label{eq:cons_bound_Q}
\end{align}
where we use $\log x \leq  \sqrt{x}$. This gives the bound of $\|\mathbf{Q}(T+1)\|_2$ that $
\|\mathbf{Q}(T+1)\|_2 \leq \mc O \big(L^2\sqrt{T} \log(L/\delta) \big)$.

Furthermore, by Lemma \ref{lem:term_III_th_diff}, we know that the key to bounding $\sum_{t=1}^{T} \|\theta^{t+1}  - \theta^{t}\|_1$ is also the drift bound for $\mathbf{Q}(t)$. Therefore, by \eqref{eq:cons_bound_Q} and the settings of the parameters $\alpha, \lambda, V$, we have
\begin{align}
\sum_{t=1}^T \|\theta^{t+1}-\theta^{t}\|_1 \leq \mc O \l( L |\mc S| \sqrt{|\mc A| T} \log ( T |\mc S| |\mc A|/\delta) \r),\label{eq:cons_delta_theta} 
\end{align}
by the facts that $L \leq |\mc S|$ , $|\mc S| > 1$, $|\mc A| > 1$ and the condition $|\mc S| |\mc A| \leq  T$. Thus combining \eqref{eq:cons_bound_Q} and \eqref{eq:cons_delta_theta} with Lemma \ref{lem:term_III_all}, and letting $\delta = \zeta / T$, then with probability at least $1-\zeta$, we have
\begin{align*}
\text{Term (IV)} \leq \mc O \l( L |\mc S| \sqrt{|\mc A| T} \log (T |\mc S| |\mc A|/\delta) \r).
\end{align*}
Combining results for Term (III) and Term (IV) with \eqref{eq:constr_decomp}, by union bound, with probability at least $1-4\zeta$, the constraint violation is bounded as $\mathrm{Violation}(T) \leq  \widetilde{\mc O} ( L|\mathcal S| \sqrt{T|\mathcal A|} ).
$ This finishes the proof.
\end{proof}

\section{Conclusion}
In this paper, we propose a new upper confidence primal-dual algorithm to solve online constrained episodic MDPs with arbitrarily varying losses and stochastically changing constraints. In particular, our algorithm does not require transition models of the MDPs and delivers an $\mathcal{O}(L|\mc S|\sqrt{|\mc A|T})$ upper bounds of both the regret and the constraint violation. 

\bibliographystyle{ims}
\bibliography{bibliography}

\newpage
\begin{appendices}
\input{supplement.tex}

\end{appendices}

\end{document}

%% file: supplement.tex

\onecolumn
\vspace{1em}
\centerline{ {\LARGE Supplementary Material} }
\renewcommand{\thesection}{\Alph{section}}

\vspace{1em}

\section{Efficient Solver for Subproblem \eqref{eq:optimistic-lp}} \label{sec:linear_solver}

In this section, we provide the details on how to efficiently solve the subproblem \eqref{eq:optimistic-lp}. We can further rewrite \eqref{eq:optimistic-lp} into the following equivalent form
\begin{align*}
\theta^{t}=\argmin_{\theta\in\Delta(\ell(t), \zeta)} ~\alpha^{-1}\langle \psi^{t-1}, \theta \rangle +  D(\theta, \widetilde{\theta}^{t-1}),
\end{align*}
where we let $\psi^{t-1} := Vf^{t-1} + \sum_{i=1}^IQ_i(t-1)g_i^{t-1}$. According to \citet{rosenberg2019online}, solving the above problem is decomposed to the following two steps
\begin{align}
\underline{\theta}^t&=\argmin_{\theta} ~\alpha^{-1}\langle \psi^{t-1}, \theta \rangle +  D(\theta, \widetilde{\theta}^{t-1}),\label{eq:subpro_step1}\\
\theta^{t}&=\argmin_{\theta\in\Delta(\ell(t), \zeta)} ~ D(\theta, \underline{\theta}^t). \label{eq:subpro_step2}
\end{align}
Note that the first step, i.e., \eqref{eq:subpro_step1},  is an unconstrained problem, which has a closed-form solution
\begin{align}
\underline{\theta}^t(s,a,s') = \widetilde{\theta}^{t-1}(s,a,s') e^{-\psi^{t-1}/\alpha}, ~~~~\forall (s,a,s')\in \mathcal S_k\times \mathcal A \times \mathcal S_{k+1}, ~~\forall k =0, \ldots, L-1. \label{eq:sol_subpro_step1}
\end{align}
The second step, i.e., \eqref{eq:subpro_step2}, can be viewed as a projection of $\underline{\theta}^t(s,a,s')$ onto the feasible set $\Delta(\ell(t), \zeta)$.  With the definition of the feasible set as in \eqref{eq:subpro_fea_set}, further by Theorem 4.2 of \citet{rosenberg2019online} and Lemma 7 of \citet{jin2019learning}, and plugging in $\underline{\theta}^t$ computed as \eqref{eq:sol_subpro_step1},  we have
\begin{align}
\theta^{t}(s,a,s') = \frac{\widetilde{\theta}^{t-1}(s,a,s')}{Z_{t}^{k(s)}(\mu^{t}, \beta^{t})} e^{B_{t}^{\mu^{t}, \beta^{t}}(s,a,s')}, \label{eq:sol_subpro_step2}
\end{align}
where $k(s)$ is a mapping from state $s$ to its associated layer index, and $Z_{t}^{k}(\mu, \beta)$ and $B_{t}^{\mu, \beta}$ are defined as follows
\begin{align*}
B_{t}^{\mu, \beta}(s,a,s') &= \mu^{-}(s,a,s') - \mu^{+}(s,a,s') + (\mu^{+}(s,a,s')+\mu^{-}(s,a,s')) \varepsilon_{\ell(t)}^{\zeta}(s,a) + \beta(s') \\
&\quad - \beta(s)-\psi^{t-1}(s,a,s') / \alpha  - \sum_{s''\in \mathcal{S}_{k(s) + 1}} \widehat{P}_{\ell(t)}(s''|s,a) (\mu^{-}(s,a,s'') - \mu^{+}(s,a,s'')),\\
Z_{t}^{k}(\mu, \beta) &= \sum_{s\in \mathcal{S}_k} \sum_{a\in \mathcal{A}} \sum_{s'\in \mathcal{S}_{k+1}}\widetilde{\theta}^{t-1}(s,a,s') e^{B_{t}^{\mu, \beta}(s,a,s')},
\end{align*}
where $\beta : \mathcal{S} \rightarrow \mathbb{R}$ and $\mu = (\mu^{+}, \mu^{-})$ with $\mu^{+}, \mu^{-}:\mathcal{S}\times\mathcal{A}\times\mathcal{S} \rightarrow \mathbb{R}_{\geq 0}$. Specifically, the variables $\mu^{t}, \beta^{t}$ in \eqref{eq:sol_subpro_step2} are obtained by solving a convex optimization with only non-negativity constraints, which is
\begin{align}
\mu^t, \beta^t = \argmin_{\mu, \beta \geq 0}\sum_{k=0}^{L-1}\log Z_{t}^{k}(\mu, \beta). \label{eq:non-neg_sol} 
\end{align}
Therefore, by solving \eqref{eq:non-neg_sol}, we can eventually compute $\theta^{t}$ by \eqref{eq:sol_subpro_step2}. Since \eqref{eq:non-neg_sol} is associated with a convex optimization with only non-negativity constraints, it can be solved much efficiently.

\section{Structure of Optimization Problem Sequence} \label{sec:A0}

We have the following simple sufficient condition which is a direct corollary of Lemma 1 in \citet{nedic2009approximate}:
\begin{lemma}
Suppose that the problem \eqref{eq:partial-static-prob} is feasible. Then, the set of Lagrange multipliers $\mathcal{V}_{t,\tau}^*$ defined in Assumption \ref{as:selm} is nonempty and bounded if the Slater condition holds, i.e., $\exists \theta\in\Delta,~\varepsilon>0$ such that $\dotp{g_i}{\theta} \leq c_i-\varepsilon,~\forall i \in \{1,\ldots,I\}$.  
\end{lemma}
In fact, it can be shown that some certain constraint qualification condition more general than Slater condition can imply the boundedness of Lagrange multipliers (see, for example, Lemma 18 of \citet{wei2019online}). Thus, Assumption \ref{as:selm} is weaker than Slater condition commonly adopted in previous constrained online learning works. The motivation for such a Lagrange multiplier condition is that it is a sufficient condition  of a key structural property on the dual function $q^{(t,\tau)}(\eta)$, namely, the error bound condition. Formally, we have the following definition:
\begin{definition}[Error Bound Condition (EBC)]\label{def:ebc}
Let $h(\mathbf{x})$ be a concave function over $\mathbf{x}\in\mathcal{X}$, where $\mathcal{X}$ is closed and convex. 
Suppose $\Lambda^*:= \argmax_{\mathbf{x}\in\mathcal{X}}h(\mathbf{x})$ is non-empty. The function $h(\mathbf{x})$ satisfies the EBC if there exists constants $\vartheta,~\sigma>0$ such that for any $\mf x\in\mathcal{X}$ satisfying\footnote{We let $\dist(\mf x, \Lambda^*) := \min_{\mf x' \in \Lambda^*} \|\mf x - \mf x' \|_2$ as the Euclidean distance between a point $\mf x$ and the set $\Lambda^*$.}  
$\dist(\mf x,\Lambda^*)\geq \vartheta$,
\begin{equation*}
 h(\mf x^*)-h(\mf x) \geq \sigma \cdot \dist(\mf x,\Lambda^*) ~~\text{ with } \mf x^* \in \Lambda^*.
\end{equation*}
\end{definition}
Note that in Definition \ref{def:ebc}, $\Lambda^*$ is a closed convex set, which follows from the fact that $h(\mf x)$ is a concave function and thus all superlevel sets are closed and convex. 
The following lemma, whose proof can be found in Lemma 5 of \citet{wei2019online}, shows the relation between the Lagrange multiplier condition and the dual function:
\begin{lemma}
Fix $T\geq 1$. Suppose Assumption \ref{as:selm} holds, then for any $t\in\{1,\ldots,T\}$ and $\tau=\sqrt{T}$, the dual function $q^{(t,\tau)}(\eta)$ satisfies the EBC with $\sigma > 0$ and $\vartheta > 0$.  
\end{lemma}
This lemma is equivalent to Lemma \ref{lem:leb} in the main text.

\section{Proofs of the Lemmas in  Section \ref{sec:regret}} \label{sec:A}

\subsection{Proof of Lemma \ref{lem:term_I}} 
We first provide Lemmas \ref{lem:num_seq_bound} and \ref{lem:err_shr} below. Then, we give the proof of Lemma \ref{lem:term_I} based on these lemmas.

\begin{lemma}[Lemma 19 in \citet{jaksch2010near}] \label{lem:num_seq_bound} For any sequence of numbers $x_1, \ldots, x_n$ with $0\leq x_k \leq X_{k-1}:=\max \big\{1, \sum_{i=1}^{k-1} x_i \big\}$ with $1\leq k \leq n$, the following inequality holds
\begin{align*}
\sum_{k=1}^n \frac{x_k}{\sqrt{X_{k-1}}} \leq (\sqrt{2} + 1)\sqrt{X_n}.
\end{align*}

\end{lemma}

\begin{lemma} \label{lem:err_shr} Let $\widehat{d}_t(s)$ and $d_t(s)$ be the state stationary distributions associated with $\theta^t$ and $\overline{\theta}^t$ respectively, and $\widehat{P}_{\ell(t)}(s'|a,s)$ and $P(s'|a,s)$ be the corresponding transition distributions. Denote $\pi_t(a|s)$ as the policy at episode $t$. There are $\theta^t(s,a,s') = \widehat{d}_t (s)\pi_t(a|s) \widehat{P}_{\ell(t)}(s'|a,s)$ and $\overline{\theta}^t(s,a,s) = d_t (s)\pi_t(a|s) P(s'|a,s)$. On the other hand, there are also $\widehat{d}_t(s')=\sum_{s \in\mc S_k}\sum_{a \in \mc A} \theta^t(s,a,s'), \forall s' \in \mc S_{k+1}$, and $d_t(s')=\sum_{s\in \mc S_k}\sum_{a \in \mc A} \theta^t(s,a,s'), \forall s' \in \mc S_{k+1}$. Then, we have the following inequality
\begin{align*}
\|\theta^t - \overline{\theta}^t\|_1 \leq\sum_{k=0}^{L-1}\sum_{j = 0}^{k}\sum_{s \in \mc S_j} \sum_{a\in \mc A} \mu_t(s,a)\|\widehat{P}_{\ell(t)}(\cdot|s,a) - P(\cdot|s,a)\|_1,
\end{align*}
where we let $\mu_t(s,a) = d_t(s) \pi_t(a|s)$.
\end{lemma}

\begin{proof} [Proof of Lemma \ref{lem:err_shr}] By the definitions of $\widehat{d}_t$, $d_t$, $\widehat{P}_{\ell(t)}$, $P$, and $\pi_t$ shown in Lemma \ref{lem:err_shr}, we have
\begin{align*}
\|\theta^t - \overline{\theta}^t\|_1 &=  \sum_{k=0}^{L-1} \sum_{s \in \mc S_k} \sum_{a \in \mc A}  \| \theta^t(a,s, \cdot ) - \overline{\theta}^t(a,s, \cdot )  \|_1 \\
&= \sum_{k=0}^{L-1} \sum_{s \in \mc S_k} \sum_{a \in \mc A}  \pi_t(a | s )   \| \widehat{P}_{\ell(t)} ( \cdot | a,s) \widehat{d}_t (s) -  P( \cdot | a,s)d_t (s)   \|_1 \\
&= \sum_{k=0}^{L-1} \sum_{s \in \mc S_k} \sum_{a \in \mc A}  \pi_t(a | s )   \| \widehat{P}_{\ell(t)} ( \cdot | a,s) \widehat{d}_t (s)  -\widehat{P} ( \cdot | a,s) d_t (s) \\
& \qquad \qquad \qquad \qquad \qquad  + \widehat{P} ( \cdot | a,s) d_t (s)  -  P( \cdot | a,s)d_t (s)\|_1 .
\end{align*}
Thus, with the above equalities, and by triangle inequality for $\|\cdot\|_1$, we can bound the term $\|\theta^t-\overline{\theta}^t\|_1$ in the following way
\begin{align}
\begin{aligned}\label{eq:prob_decomp}
\|\theta^t - \overline{\theta}^t\|_1 &\leq   \sum_{k=0}^{L-1} \sum_{s \in \mc S_k} \sum_{a \in \mc A}  \pi_t(a | s )  [ \| \widehat{P}_{\ell(t)} ( \cdot | a,s) \widehat{d}_t (s)  -\widehat{P}_{\ell(t)} ( \cdot | a,s) d_t (s) \|_1 \\
& \qquad  + \|\widehat{P}_{\ell(t)} ( \cdot | a,s) d_t (s)  -  P( \cdot | a,s)d_t (s)   \|_1 ] \\
& \leq  \sum_{k=0}^{L-1} \sum_{s \in \mc S_k} \sum_{a \in \mc A}  \pi_t(a | s )  d_t (s)   \| \widehat{P}_{\ell(t)} ( \cdot | a,s) -  P( \cdot | a,s)\|_1 \\
&\qquad +   \sum_{k=0}^{L-1} \sum_{s \in \mc S_k} \sum_{a \in \mc A}  \pi_t(a | s )   \| \widehat{P}_{\ell(t)} ( \cdot | a,s)\|_1  \cdot  |\widehat{d}_t (s)  -d_t (s)|.
\end{aligned}
\end{align}
Then we need to bound the last two terms of \eqref{eq:prob_decomp} respectively. For the first term of RHS in \eqref{eq:prob_decomp}, we have
\begin{align}
\begin{aligned}\label{eq:prob_decomp_1}
&\sum_{k=0}^{L-1} \sum_{s \in \mc S_k} \sum_{a \in \mc A}  \pi_t(a | s )  d_t (s)   \| \widehat{P}_{\ell(t)} ( \cdot | a,s) -  P( \cdot | a,s)\|_1 \\ 
&\qquad =  \sum_{k=0}^{L-1} \sum_{s \in \mc S_k} \sum_{a \in \mc A}   \mu_t (s, a)   \| \widehat{P}_{\ell(t)} ( \cdot | a,s) -  P( \cdot | a,s)\|_1,
\end{aligned}
\end{align}
since $\mu_t (s, a) = \pi_t(a | s )  d_t (s)$ denotes the joint distribution probability of $(s,a)$.

Next, we bound the last term of RHS in \eqref{eq:prob_decomp}, which is
\begin{align*}
\sum_{k=0}^{L-1} \sum_{s \in \mc S_k} \sum_{a \in \mc A}  \pi_t(a | s )   \| \hat{P}_{\ell(t)} ( \cdot | a,s)\|_1  \cdot  |\widehat{d}_t (s)  -d_t (s)|  = \sum_{k=0}^{L-1} \sum_{s \in \mc S_k} \sum_{a \in \mc A}  \pi_t(a | s )  |\widehat{d}_t (s)  -d_t (s)|, 
\end{align*}
since $ \| \hat{P}_{\ell(t)} ( \cdot | a,s)\|_1 = \sum_{s'\in \mc S_{k+1}} \hat{P}_{\ell(t)} ( s' | a,s) = 1$. Furthermore, we can bound the last term above as
\begin{align*}
&\sum_{k=0}^{L-1} \sum_{s \in \mc S_k} \sum_{a \in \mc A}  \pi_t(a | s )  |\widehat{d}_t (s)  -d_t (s)| \\
&\qquad =\sum_{k=0}^{L-1} \sum_{s\in \mc S_k} |\widehat{d}_t (s)  -d_t (s)| \\
&\qquad =\sum_{k=1}^{L-1} \sum_{s\in \mc S_k} |\widehat{d}_t (s)  -d_t (s)| \\
&\qquad =\sum_{k=1}^{L-1} \sum_{s\in \mc S_k}  \Big|\sum_{s'' \in \mc S_{k-1}} \sum_{a\in \mc A}\theta^t (s'', a, s)  - \sum_{s'' \in \mc S_{k-1}} \sum_{a\in \mc A} \overline{\theta}^t (s'', a, s) \Big|, 
\end{align*}
where the first equality is due to $\sum_{a\in \mc A} \pi_t(a | s) = 1$, the second equality is due to $\widehat{d}_t(s_0) = d_t(s_0)  = 1$, and the third equality is by the relations $\widehat{d}_t(s)=\sum_{s'' \in \mc S_{k-1}}\sum_{a \in \mc A} \theta^t(s'',a,s)$ and $d_t(s)=\sum_{s'' \in \mc S_{k-1}}\sum_{a \in \mc A} \overline{\theta}^t(s'',a',s)$, $\forall s \in \mc S_k$. Further bounding the last term of the above equation gives
\begin{align*}
&\sum_{k=1}^{L-1} \sum_{s\in \mc S_k}  \bigg|\sum_{s'' \in \mc S_{k-1}}\sum_{a \in \mc A} \theta^t (s'', a, s)  - \sum_{s'' \in \mc S_{k-1}}\sum_{a \in \mc A}  \overline{\theta}^t (s'', a, s) \bigg| \\
&\qquad \leq \sum_{k=1}^{L-1} \sum_{s\in \mc S_k} \sum_{s'' \in \mc S_{k-1}}\sum_{a \in \mc A}   \bigg|\theta^t (s'', a, s)  - \overline{\theta}^t (s'', a, s) \bigg| \\
&\qquad = \sum_{k=1}^{L-1} \sum_{s'' \in \mc S_{k-1}}\sum_{a \in \mc A}  \big\|\theta^t (s'', a,\cdot )  - \overline{\theta}^t (s'', a, \cdot ) \big\|_1 \\
&\qquad = \sum_{k=0}^{L-2} \sum_{s\in \mc S_k} \sum_{a\in \mc A}  \big\|\theta^t (s, a,\cdot )  - \overline{\theta}^t (s, a, \cdot ) \big\|_1, 
\end{align*}
which eventually implies that the last term on RHS of \eqref{eq:prob_decomp} can be bounded as
\begin{align}
\sum_{k=0}^{L-1} \sum_{s \in \mc S_k} \sum_{a \in \mc A}  \pi_t(a | s )   \| P ( \cdot | a,s)\|_1  \cdot  |\widehat{d}_t (s)  -d_t (s)| \leq \sum_{k=0}^{L-2} \sum_{s\in \mc S_k} \sum_{a\in \mc A}  \big\|\theta^t (s, a,\cdot )  - \overline{\theta}^t (s, a, \cdot ) \big\|_1. \label{eq:prob_decomp_2}
\end{align}
Therefore, plugging the bounds \eqref{eq:prob_decomp_1} and \eqref{eq:prob_decomp_2} in \eqref{eq:prob_decomp}, we have
\begin{align*}
\|\theta^t - \bar{\theta}^t\|_1 = & \sum_{k=0}^{L-1} \sum_{s \in \mc S_k} \sum_{a \in \mc A}  \big\| \theta^t(a,s, \cdot ) - \overline{\theta}^t(a,s, \cdot )  \big\|_1 \\
\leq & \sum_{k=0}^{L-1} \sum_{s \in \mc S_k} \sum_{a \in \mc A}   \mu_t (s, a)   \big\| \widehat{P}_{\ell(t)} ( \cdot | a,s) -  P( \cdot | a,s) \big\|_1 \\
&+ \sum_{k=0}^{L-2} \sum_{s\in \mc S_k} \sum_{a\in \mc A}  \big\|\theta^t (s, a,\cdot )  - \overline{\theta}^t (s, a, \cdot ) \big\|_1. 
\end{align*}
Recursively applying the above inequality, we obtain 
\begin{align*}
\|\theta^t - \overline{\theta}^t\|_1 \leq \sum_{k=0}^{L-1}\sum_{j = 0}^{k}\sum_{s \in \mc S_j} \sum_{a\in \mc A} \mu_t(s,a) \Big\|\widehat{P}_{\ell(t)}(\cdot|s,a) - P(\cdot|s,a)\Big\|_1,
\end{align*}
which completes the proof.
\end{proof}

Now, we are in position to give the proof of Lemma \ref{lem:term_I}.


\begin{proof} [Proof of Lemma \ref{lem:term_I}]
The proof for Lemma \ref{lem:term_I} adopts similar ideas in \citet{neu2012adversarial, rosenberg2019online}.

We already know $\widehat{P}_{\ell(t)}(s'|s,a)=\frac{\theta^t(s,a,s')}{\sum_{s'\in S_{k+1}}\theta^t(s,a,s')}$ and 
$\mu_t(s,a)=\sum_{s' \in\mc S_{k+1}}\theta^t(s,a,s')$, $\forall s \in \mc S_k, a \in \mc A, s' \in \mc S_{k+1}$, $\forall k \in \{0,\ldots, L-1\}$. By Lemma \ref{lem:err_shr}, one can show that
\begin{align*}
\|\theta^t - \overline{\theta}^t\|_1\leq& \sum_{k=0}^{L-1}\sum_{j = 0}^{k}\sum_{s \in \mc S_j} \sum_{a\in \mc A}
\mu_t(s,a)\big\|\widehat{P}_{\ell(t)}(\cdot|s,a) - P(\cdot|s,a)\big\|_1\\
=&\sum_{k=0}^{L-1}\sum_{j = 0}^{k}\sum_{s \in \mc S_j} \sum_{a\in \mc A}
\big[(\mu_t(s,a) - \mathbb I\{s^t_j = s,~a^t_j = a\})\big\|\widehat{P}_{\ell(t)}(\cdot|s,a) - P(\cdot|s,a)\big\|_1\\
&\qquad + \mathbb I\{s^t_j = s,~a^t_j = a\}\big\|\widehat{P}_{\ell(t)}(\cdot|s,a) - P(\cdot|s,a)\big\|_1 \big],
\end{align*}
where we denote $\mathbb I\{s^t_j = s,~a^t_j = a\})$ as the indicator random variable that equals $1$ with probability $\mu_t(s,a), \forall s \in S_j, a\in \mc A $ and $0$ otherwise. Denote $\xi^t(s,a) = \|\widehat{P}_{\ell(t)}(\cdot|s,a) - P(\cdot|s,a)\|_1$ for abbreviation. We can see that $\xi^t(s,a) \leq \|\widehat{P}_{\ell(t)}(\cdot|s,a)\|_1 + \|P(\cdot|s,a)\|_1 = 2$. Summing both sides of the above inequality over $T$ time slots, we obtain
\begin{align}
\begin{aligned} \label{eq:theta_diff_0}
\sum_{t=1}^{T}\|\theta^t - \overline{\theta}^t\|_1 &\leq \sum_{t=1}^{T} \sum_{k=0}^{L-1}\sum_{j = 0}^{k}
\sum_{s \in \mc S_j} \sum_{a \in \mc A}  (\mu_t(s,a) - \mathbb I\{s^t_j = s,~a^t_j = a\})\xi^t(s,a)\\
&\qquad +\sum_{t=1}^{T} \sum_{k=0}^{L-1}\sum_{j = 0}^{k}
\sum_{s \in \mc S_j} \sum_{a \in \mc A} [ \mathbb I\{s^t_j = s,~a^t_j = a\}\xi^t(s,a). 
\end{aligned}
\end{align}
Next, we bound the first term on RHS of \eqref{eq:theta_diff_0}. Let $\mc F^{t-1}$ be the system history up to $(t-1)$-th episode. Then, by the definition of $\mathbb I(\cdot, \cdot )$, we have
\[
\E \Big\{ \sum_{s \in \mc S_j} \sum_{a \in \mc A}(\mu_t(s,a) - \mathbb I\{s^t_j = s,~a^t_j = a\})\xi^t(s,a)~\Big|~\mc F^{t-1} \Big \} = 0,
\]
since $\xi^t$ is only associated with system randomness history up to $t-1$ episodes. Thus, the term $\sum_{s \in \mc S_j} \sum_{a \in \mc A}(\mu_t(s,a) - \mathbb I\{s^t_j = s,~a^t_j = a\})\xi^t(s,a)$ is a martingale difference sequence with respect to $\mc F^{t-1}$. Furthermore, by $\xi^t(s,a) \leq 2$ and $\sum_{s \in \mc S_j} \sum_{a \in \mc A} \mathbb I\{s^t_j = s,~a^t_j = a\}) = 1$, there will be
\begin{align*}
&\Bigg|\sum_{s \in \mc S_j} \sum_{a \in \mc A}(\mu_t(s,a) - \mathbb I\{s^t_j = s,~a^t_j = a\})\xi^t(s,a)\Bigg|\\
&\qquad \leq \Bigg| \sum_{s \in \mc S_j} \sum_{a \in \mc A}  \mathbb I\{s^t_j = s,~a^t_j = a\} \Bigg|\xi^t(s,a) + \Bigg| \sum_{s \in \mc S_j} \sum_{a \in \mc A}\mu_t(s,a)\Bigg|\xi^t(s,a) \leq 4.
\end{align*}
Thus, by Azuma's inequality, we obtain that with probability at least $1-\zeta/L$,
\[
\sum_{t=1}^{T}\sum_{s \in \mc S_j} \sum_{a \in \mc A}(\mu_t(s,a) - \mathbb I\{s^t_j = s,~a^t_j = a\})\xi^t(s,a)
\leq 4\sqrt{2T\log\frac{L}{\zeta}}.
\]
According to union bound, we further have that with probability at least $1-\zeta$, the above inequality holds for all $j = 0,...,L-1$. This implies that with probability at least $1-\zeta$, the following inequality holds
\begin{align}
\sum_{t=1}^{T} \sum_{k=0}^{L-1}\sum_{j = 0}^{k}\sum_{s \in \mc S_j} \sum_{a \in \mc A}
(\mu_t(s,a) - \mathbb I\{s^t_j = s,~a^t_j = a\}) \xi^t(s,a) \leq  2L^2\sqrt{2T\log\frac{L}{\zeta}}. \label{eq:theta_diff_1}
\end{align}
On the other hand, we adopt the same argument as the first part of the proof of Lemma 5 in \citet{neu2012adversarial} to show the upper bound of $\sum_{t=1}^{T} \sum_{k=0}^{L-1}\sum_{j = 0}^{k} \sum_{s \in \mc S_j} \sum_{a \in \mc A} \mathbb I\{s^t_j = s,~a^t_j = a\}\xi^t(s,a)$ in \eqref{eq:theta_diff_0}. Recall that $\ell(t)$ denotes the epoch that the $t$-th episode belongs to. By the definition of the state-action pair counter $N_{\ell}(s,a)$ and $n_{\ell}(s,a)$, we have
\begin{align*}
N_{\ell}(s,a) = \sum_{q=0}^{\ell-1} n_{q}(s,a).
\end{align*}
According to Lemma \ref{lem:num_seq_bound}, we have
\begin{align}
\sum_{q=1}^{\ell(t)} \frac{n_q(s,a)}{\max\{1, \sqrt{N_q(s,a)}\}} \leq (\sqrt{2} + 1) \sqrt{\sum_{q=1}^{\ell(t)}n_q(s,a) }.\label{eq:counter_bound}
\end{align}
Since we can rewrite 
\begin{align*}
&\sum_{t=1}^{T} \sum_{k=0}^{L-1}\sum_{j = 0}^{k} \sum_{s \in \mc S_j} \sum_{a \in \mc A} \mathbb I\{s^t_j = s,~a^t_j = a\}\xi^t(s,a)\\
&\qquad  =  \sum_{t=1}^{T} \sum_{k=0}^{L-1}\sum_{j = 0}^{k} \|\widehat{P}_{\ell(t)}(\cdot|s^t_j,a^t_j) - P(\cdot|s^t_j,a^t_j)\|_1,
\end{align*}
then by Lemma \ref{lem:confidence}and $T+1\leq 2T$, the following holds with probability at least $1-\zeta$, 
\begin{align*}
&\sum_{t=1}^{T} \sum_{k=0}^{L-1}\sum_{j = 0}^{k} \sum_{s \in \mc S_j} \sum_{a \in \mc A} \mathbb I\{s^t_j = s,~a^t_j = a\}\xi^t(s,a)\\
& \qquad \leq\sum_{k=0}^{L-1}\sum_{j = 0}^{k} \sum_{t=1}^{T}  \sqrt{\frac{2|\mc S_{j+1}|\log(2T|\mc S||\mc A|/\zeta)}{\max\{1,N_{\ell(t)}(s^t_j,a^t_j)\}}}\\
& \qquad \leq \sum_{k=0}^{L-1}\sum_{j = 0}^{k} \sum_{q=1}^{\ell(T)}\sum_{s\in\mathcal{S}_j}\sum_{a \in\mathcal{A}} n_q(s,a)\sqrt{\frac{2|\mc S_{j+1}|\log(2T|\mc S||\mc A|/\zeta)}{\max\{1,N_{q}(s,a)\}}}\\
& \qquad \leq  \sum_{k=0}^{L-1}\sum_{j = 0}^{k} \sum_{s\in\mathcal{S}_j}\sum_{a \in\mathcal{A}} (\sqrt{2} + 1)  \sqrt{2 \sum_{q=1}^{\ell(T)}n_q(s,a) |\mc S_{j+1}|\log \frac{2T|\mc S||\mc A|}{\zeta}},
\end{align*}
where the first inequality is due to Lemma \ref{lem:confidence}, the second inequality is by the definition of the counters $n_q(s,a)$ and $N_q(s,a)$, and the last inequality is by \eqref{eq:counter_bound}. Thus, further bounding the last term of the above inequality yields 
\begin{align*}
&\sum_{k=0}^{L-1}\sum_{j = 0}^{k} \sum_{s\in\mathcal{S}_j}\sum_{a \in\mathcal{A}}   (\sqrt{2} + 1)  \sqrt{2 \left[\sum_{q=1}^{\ell(T)}n_q(s,a)\right] |\mc S_{j+1}|\log \frac{2T|\mc S||\mc A|}{\zeta}}\\
&\qquad \leq \sum_{k=0}^{L-1}\sum_{j = 0}^{k}  (\sqrt{2} + 1)  \sqrt{2 \sum_{s\in\mathcal{S}_j}\sum_{a \in\mathcal{A}}   \left[\sum_{q=1}^{\ell(T)}n_q(s,a)\right] |\mc S_j| |\mc S_{j+1}| |\mathcal{A}|\log \frac{2T|\mc S||\mc A|}{\zeta}} \\
&\qquad \leq\sum_{k=0}^{L-1}\sum_{j = 0}^{k}  (\sqrt{2} + 1)  \sqrt{2 T  |\mc S_j| |\mc S_{j+1}| |\mathcal{A}|\log \frac{2T|\mc S||\mc A|}{\zeta}} \\
&\qquad \leq (\sqrt{2} + 1) L |\mathcal{S}|  \sqrt{2 T  |\mathcal{A}|\log \frac{2T|\mc S||\mc A|}{\zeta} },
\end{align*}
where the first inequality is due to Jensen's inequality, the second inequality is due to $\sum_{s\in\mathcal{S}_j}\sum_{a \in\mathcal{A}}  \sum_{q=1}^{\ell(T)} \allowbreak n_q(s,a) \leq T$, and the last inequality is by bounding the term $\sum_{k=0}^{L-1}\sum_{j = 0}^{k}  \sqrt{|\mc S_j| |\mc S_{j+1}|}  \leq \sum_{k=0}^{L-1}\sum_{j = 0}^{k} \allowbreak (|\mc S_j|  + |\mc S_{j+1}|)/2 \leq L |\mathcal{S}|$. The above results imply that with probability at least $1-\zeta$,
\begin{align}
\sum_{t=1}^{T} \sum_{k=0}^{L-1}\sum_{j = 0}^{k} \sum_{s \in \mc S_j} \sum_{a \in \mc A} \mathbb I\{s^t_j = s,~a^t_j = a\}\xi^t(s,a) \leq (\sqrt{2} + 1) L |\mathcal{S}|  \sqrt{2 T  |\mathcal{A}|\log \frac{2T|\mc S||\mc A|}{\zeta}}. \label{eq:theta_diff_2}
\end{align}
By union bound, combining \eqref{eq:theta_diff_0}, \eqref{eq:theta_diff_1} and \eqref{eq:theta_diff_2}, we obtain with probability at least $1-2\zeta$, 
\begin{align*}
\sum_{t=1}^{T}\|\theta^t - \overline{\theta}^t\|_1\leq
(\sqrt 2+1)L|\mc S|\sqrt{2T|\mc A|\log\frac{2T|\mc S||\mc A|}{\zeta}} + 2L^2\sqrt{2T\log\frac{L}{\zeta}}. 
\end{align*} 
This completes the proof.
\end{proof}

\subsection{Proof of Lemma \ref{lem:regret_primal_bound}}

We provide Lemmas \ref{lem:breg_tria}, \ref{lem:pins}, and \ref{lem:breg_diff} first.  Based on them, we give the proof of Lemma \ref{lem:regret_primal_bound}.

\begin{lemma} [Lemma 14 in \citet{wei2019online}]\label{lem:breg_tria} Let $M$ and $M^o$ denote the probability simplex and the set of the probability simplex excluding the boundary respectively. Assuming $\mathbf{y} \in M^o$, and letting $\mc C \subseteq M$, then the following inequality holds
\begin{align*} 
h(\mathbf{x}^{opt}) + \alpha D(\mathbf{x}^{opt}, \mathbf{y}) \leq h(\mathbf{z}) + \alpha D(\mathbf{z},\mathbf{y})-\alpha D(\mathbf{z},\mathbf{x}^{opt}), ~~\forall \mathbf{z} \in \mc C,
\end{align*}
where $\mathbf{x}^{opt} \in \arg\min_{\mathbf{x}\in \mc C} h(\mathbf{x})+\alpha D(\mathbf{x},\mathbf{y})$, $h(\cdot)$ is a convex function, and $D(\cdot, \cdot)$ is the unnormalized KL divergence for two distributions.
\end{lemma}
Lemma \ref{lem:breg_tria} is an extension of Lemma 14 in \citet{wei2019online}, whose proof follows the one in \citet{wei2019online}. Specifically, the unnormalized KL divergence is a special case of the Bregman divergence studied in \citet{wei2019online}.

\begin{lemma} \label{lem:pins} For any $\theta$ and $\theta'$ satisfying $\sum_{s\in \mc S_k} \sum_{a\in\mc A} \sum_{s'\in\mc S_{k+1}} \theta(s,a,s')=1, \text{ and }~ \theta(s,a,s')\geq 0, \forall k \in \{0, \ldots, L-1\}$ and $\sum_{s\in\mc S_k} \sum_{a\in \mc A} \theta(s, a, s') = \sum_{a\in \mc A} \sum_{s''\in\mc S_{k+2}}  \theta(s', a, s''), \forall s' \in \mc S_{k+1}, \forall k\in \{ 0, \ldots, L-2 \}$, we let $\theta_k:= [\theta(s, a , s')]_{s \in \mc S_k, a \in \mc A, s' \in \mc S_{k+1}}$ denote the vector formed by the elements $\theta(s, a , s')$ for all $s \in \mathcal{S}_k, a \in \mathcal{A}, s \in \mathcal{S}_{k+1}$. We also let $\theta'_k:= [\theta'(s, a, s')]_{s \in \mc S_k, a \in \mc A, s' \in \mc S_{k+1}}$ similarly denote a vector formed by $\theta'(s, a, s')$. Then, we have
\begin{align*}
D(\theta, \theta') \geq \frac{1}{2}\sum_{k=0}^{L-1}\|\theta_k - \theta'_k\|_1^2 \geq \frac{1}{2L}\|\theta - \theta'\|_1^2,
\end{align*}
where $D(\cdot, \cdot)$ is defined as in \eqref{eq:un-KL}.
\end{lemma}

\begin{proof} [Proof of Lemma \ref{lem:pins}] We prove the lemma by the following inequality
\begin{align*} 
D(\theta, \theta')&= \sum_{k=0}^{L-1}\sum_{s \in \mc S_k} \sum_{a \in \mc A } \sum_{ s \in \mc S_{k+1}} \theta(s, a , s') \frac{ \theta(s, a , s') }{ \theta'(s, a , s')}  - \theta(s, a , s') + \theta'(s, a , s')\\
&= \sum_{k=0}^{L-1}\sum_{s \in \mc S_k} \sum_{a \in \mc A } \sum_{ s \in \mc S_{k+1}} \theta(s, a , s') \frac{ \theta(s, a , s') }{ \theta'(s, a , s')} \\
&\geq  \frac{1}{2}\sum_{k=0}^{L-1}\|\theta_k - \theta'_k\|_1^2 \geq \frac{1}{2L} \bigg( \sum_{k=0}^{L-1} \|\theta_k - \theta'_k \|_1   \bigg)^2 \geq \frac{1}{2L}\|\theta - \theta'\|_1^2,
\end{align*}
where the inequality is due to the Pinsker's inequality since $\theta_k$ and $\theta'_k$ are two probability distributions such that $\|\theta_k\|_1 = 1$ and $\|\theta'_k\|_1=1$. This completes the proof.
\end{proof}

\begin{lemma}\label{lem:breg_diff}
For any $\theta$ and $\theta'$ satisfying $\sum_{s\in \mc S_k} \sum_{a\in\mc A} \sum_{s'\in\mc S_{k+1}} \theta(s,a,s')=1, \text{ and }~ \theta(s,a,s')\geq 0, \forall k \in \{0, \ldots, L-1\}$ and $\sum_{s\in\mc S_k} \sum_{a\in \mc A} \theta(s, a, s') = \sum_{a\in \mc A} \sum_{s''\in\mc S_{k+2}}  \theta(s', a, s''), \forall s' \in \mc S_{k+1}, \forall k\in \{ 0, \ldots, L-2 \}$, letting $ \widetilde{\theta}'(s,a,s') = (1-\lambda) \theta'(s,a,s') +  \frac{\lambda}{|\mathcal{A}||\mathcal{S}_k| |\mathcal{S}_{k+1}|}, \forall (s,a,s')\in \mc S_k \times \mc A \times \mc S_{k+1}, \forall k = 1,\ldots, L-1$ with $0 < \lambda \leq 1$, then we have
\begin{align*}
&D(\theta, \widetilde{\theta}' ) - D(\theta, \theta') \leq \lambda L \log |\mathcal{S}|^2 |\mathcal{A}|,\\
&D(\theta, \widetilde{\theta}') \leq  L \log \Big( \frac{|\mathcal{S}|^2 |\mathcal{A}|}{\lambda} \Big).
\end{align*}
\end{lemma}

\begin{proof} [Proof of Lemma \ref{lem:breg_diff}] We start our proof as follows
\begin{align*}
D(\theta, \widetilde{\theta}' ) - D(\theta, \theta') 
&= \sum_{k=0}^{L-1}\sum_{s \in \mc S_k} \sum_{a \in \mc A } \sum_{ s \in \mc S_{k+1}} \theta(s, a , s') \Big( \log \frac{\theta(s, a , s')}{\widetilde{\theta}'(s, a, s')} - \log \frac{\theta(s, a , s')}{\theta'(s, a, s')} \Big) \\
&\quad + \widetilde{\theta}'(s, a, s') - \theta'(s, a , s')\\
& = \sum_{k=0}^{L-1}\sum_{s \in \mc S_k} \sum_{a \in \mc A } \sum_{ s \in \mc S_{k+1}} \theta(s, a , s') \big( \log \theta'(s, a , s')- \log \widetilde{\theta}'(s, a, s')\big)\\
& = \sum_{k=0}^{L-1}\sum_{s \in \mc S_k} \sum_{a \in \mc A } \sum_{ s \in \mc S_{k+1}} \theta(s, a , s') \Big( \log \theta'(s, a , s') \\
&\quad - \log [(1-\lambda) \theta'(s,a,s') +  \lambda/|\mathcal{A}||\mathcal{S}_k| |\mathcal{S}_{k+1}|]\Big),
\end{align*}
where the last equality is by substituting $\widetilde{\theta}'(s,a,s') = (1-\lambda) \theta'(s,a,s') +  \frac{\lambda}{|\mathcal{A}||\mathcal{S}_k| |\mathcal{S}_{k+1}|}, \forall (s,a,s')\in \mc S_k \times \mc A \times \mc S_{k+1}, \forall k = 1,\ldots, L-1$. Thus, by bounding the last term above, we further have
\begin{align*}
D(\theta, \widetilde{\theta}' ) - D(\theta, \theta') &\leq  \sum_{k=0}^{L-1}\sum_{s \in \mc S_k} \sum_{a \in \mc A } \sum_{ s \in \mc S_{k+1}} \theta(s, a , s') \bigg( \log \theta'(s, a , s')\\
&\quad - (1-\lambda)\log \theta'(s, a , s')- \lambda \log  \frac{1}{|\mathcal{S}_k| |\mathcal{S}_{k+1}| |\mathcal{A}|} \bigg) \\
& = \sum_{k=0}^{L-1}\sum_{s \in \mc S_k} \sum_{a \in \mc A } \sum_{ s \in \mc S_{k+1}} \lambda  \theta(s, a , s') \big( \log \theta'(s, a , s')+  \log (|\mathcal{S}_k| |\mathcal{S}_{k+1}| |\mathcal{A}|) \big) \\
&\leq \sum_{k=0}^{L-1}\sum_{s \in \mc S_k} \sum_{a \in \mc A } \sum_{ s \in \mc S_{k+1}} \lambda  \theta(s, a , s') \log (|\mathcal{S}_k| |\mathcal{S}_{k+1}| |\mathcal{A}|) \leq \lambda L \log |\mathcal{S}|^2 |\mathcal{A}|,
\end{align*}
where the first inequality is by Jensen's inequality and the second inequality is due to $\log \theta'(s, a , s')  \leq 0$ since $0 < \theta(s, a , s') \leq 1$, and the last inequality is due to H\"older's inequality that $\langle \mf x, \mf y \rangle \leq \|\mf x\|_1 \|\mf y\|_\infty$ and $|\mathcal{S}_k| |\mathcal{S}_{k+1}| \leq |\mathcal{S}|^2$.

Moreover, we have
\begin{align*}
D(\theta, \widetilde{\theta}' ) &= \sum_{k=0}^{L-1}\sum_{s \in \mc S_k} \sum_{a \in \mc A } \sum_{ s \in \mc S_{k+1}} \theta(s, a , s') \log \frac{\theta(s, a , s')}{\widetilde{\theta}'(s, a, s')}  - \theta(s, a , s') + \theta'(s, a , s')\\
&= \sum_{k=0}^{L-1}\sum_{s \in \mc S_k} \sum_{a \in \mc A } \sum_{ s \in \mc S_{k+1}} \theta(s, a , s') \big( \log \theta(s, a , s') - \log \widetilde{\theta}'(s, a, s')\big) \\
&= \sum_{k=0}^{L-1}\sum_{s \in \mc S_k} \sum_{a \in \mc A } \sum_{ s \in \mc S_{k+1}} \theta(s, a , s') \big( \log \theta(s, a , s')  -\log[ (1-\lambda)\theta'(s, a , s')+ \lambda   /(|\mathcal{S}_k| |\mathcal{S}_{k+1}| |\mathcal{A}|) ]\big) \\
& \leq - \sum_{k=0}^{L-1}\sum_{s \in \mc S_k} \sum_{a \in \mc A } \sum_{ s \in \mc S_{k+1}} \theta(s, a , s') \big(  \log[ (1-\lambda)\theta'(s, a , s')+ \lambda   /(|\mathcal{S}_k| |\mathcal{S}_{k+1}| |\mathcal{A}|) ] \big) \\
& \leq - \sum_{k=0}^{L-1}\sum_{s \in \mc S_k} \sum_{a \in \mc A } \sum_{ s \in \mc S_{k+1}}   \theta(s, a , s') \cdot  \log\frac{ \lambda}{|\mathcal{S}_k| |\mathcal{S}_{k+1}| |\mathcal{A}| } \leq  L \log \frac{|\mathcal{S}|^2 |\mathcal{A}|}{\lambda} ,
\end{align*}
where the first inequality is due to $\log \theta(s,a,s') \leq 0$, the second inequality is due to the monotonicity of logarithm function, and the third inequality is by  as well as $|\mathcal{S}_k| |\mathcal{S}_{k+1}| \leq |\mathcal{S}|^2$. This completes the proof.
\end{proof}

Now we are ready to provide the proof of Lemma \ref{lem:regret_primal_bound}.
\begin{proof}[Proof of Lemma \ref{lem:regret_primal_bound}] 
First of all, by Lemma \ref{lem:confidence}, we know that 
\[
\|P(\cdot|s,a) -\widehat P_\ell(\cdot|s,a)\|_1\leq \varepsilon_\ell^\zeta(s,a), 
\]
with probability at least $1-\zeta$, for all epochs $\ell$ and any state and action pair $(s,a)\in\mc S\times\mc A$. Thus, we have that for any epoch $\ell \leq \ell(T+1)$, 
\begin{align*}
\Delta \subseteq \Delta(\ell, \zeta)
\end{align*} 
holds with probability at least $1-\zeta$. 

This can be easily proved in the following way: If any $\overline{\theta} \in \Delta$, then for all $k=\{0,\ldots, T-1\}$, $s\in \mc S_k$ and $a \in \mc A$, 
\begin{align*}
\frac{\overline{\theta}(s,a,\cdot)}{\sum_{s'\in \mc S_{k+1}}\overline{\theta}(s,a,s')} = P(\cdot | s, a).
\end{align*}
Then, we obtain with probability at least $1-\zeta$,
\begin{align*}
&\Big\|\frac{\overline{\theta}(s,a,\cdot)}{\sum_{s'\in \mc S_{k+1}}\overline{\theta}(s,a,s')}  - \widehat{P}_\ell(\cdot | s,a)\Big\|_1 \\
&\qquad\leq \Big\|\frac{\overline{\theta}(s,a,\cdot)}{\sum_{s'\in \mc S_{k+1}}\overline{\theta}(s,a,s')}  - P(\cdot | s,a)\Big\|_1 + \Big\|P(\cdot | s,a)  - \widehat{P}_\ell(\cdot | s,a)\Big\|_1 \leq 0 + \varepsilon_\ell^\zeta(s,a) \leq \varepsilon_\ell^\zeta(s,a),
\end{align*}
where the inequality is by Lemma \ref{lem:confidence}. Therefore, we know that $\overline{\theta} \in \Delta(\ell , \zeta)$, which proves the above claim.

We define the event $\mathcal{D}_{T}$ as follows:
\begin{align}
\text{Event } \mathcal{D}_{T} : \Delta \subseteq \cap_{\ell=1}^{\ell(T+1)} \Delta(\ell, \zeta), \label{eq:event_opt_feas}
\end{align}
by which we have
\begin{align*}
\Pr (\mathcal{D}_{T}) \geq 1- \zeta.
\end{align*}
Thus, for any
$\overline\theta^*$ that is a solution to problem \eqref{eq:prob-3}, we have $\overline\theta^*\in \Delta$. If event $\mc D_T$ happens, then $\overline\theta^*\in \cap_{\ell=1}^{\ell(T)} \Delta(\ell, \zeta)$. Now we have that the updating rule of $\theta$ follows $\theta^t = \arg\min_{\theta\in \Delta(\ell(t), \zeta)} \big\langle Vf^{t-1} + \sum_{i=1}^IQ_i(t-1)g_i^{t-1}, \theta \big \rangle  + \alpha D(\theta,  \widetilde{\theta}^{t-1})$ as shown in \eqref{eq:optimistic-lp}, and also $ \overline\theta^*\in \cap_{\ell=1}^{\ell(T+1)} \Delta(\ell, \zeta)$ holds with probability at least $1-\zeta$. According to Lemma \ref{lem:breg_tria}, letting $\mathbf{x}^{opt} = \theta^t$, $\mathbf{z} = \overline{\theta}^*$, $\mathbf{y} = \widetilde{\theta}^{t-1}$ and $h(\theta) = \big\langle Vf^{t-1} + \sum_{i=1}^IQ_i(t-1)g_i^{t-1}, \theta\big 
\rangle$, we have that with probability at least $1-\zeta$, the following holds for all episodes $t= 2,\ldots,T+1$
\begin{align}
\begin{aligned}\label{eq:update_breg_tria}
&\Big\langle Vf^{t-1} + \sum_{i=1}^IQ_i(t-1)g_i^{t-1}, \theta^t \Big\rangle + \alpha D(\theta^t, \widetilde{\theta}^{t-1}) \\
&\qquad  \leq \Big\langle Vf^{t-1} + \sum_{i=1}^IQ_i(t-1)g_i^{t-1},  \overline{\theta}^* \Big\rangle + \alpha D(\overline{\theta}^* , \widetilde{\theta}^{t-1})-\alpha D(\overline{\theta}^*, \theta^t ), 
\end{aligned}
\end{align}
which means once given the event $\mc D_T$ happens, the inequality \eqref{eq:update_breg_tria} will hold. 

On the other hand, according to the updating rule of $\mathbf{Q}(\cdot)$ in \eqref{eq:Q-update}, which is $Q_i(t) = \max\{Q_i(t-1) + \langle g_i^{t-1}, \theta^t \rangle - c_i,~0\}$, we know that
\begin{align*}
Q_i(t)^2 \leq \l(\max\{Q_i(t-1) + \dotp{g_i^{t-1}}{\theta^t} - c_i,~0\}\r)^2\leq \l(Q_i(t-1) + \dotp{g_i^{t-1}}{\theta^t} - c_i \r)^2,
\end{align*} 
which further leads to
\begin{align*}
Q_i(t)^2-Q_i(t-1)^2 \leq&  2Q_i(t-1) \l( \dotp{g_i^{t-1}}{\theta^t} - c_i \r) + \l(\dotp{g_i^{t-1}}{\theta^t} - c_i \r)^2.
\end{align*}
Taking summation on both sides of the above inequality from $i=1$ to $I$, we have
\begin{align}
\begin{aligned}\label{eq:Q_difference}
&\frac{1}{2} \l( \|\mathbf{Q}(t)\|^2-\|\mathbf{Q}(t-1)\|^2 \r) \\  
&\qquad \leq \sum_{i=1}^I  \dotp{Q_i(t-1) g_i^{t-1}}{\theta^t}  - \sum_{i=1}^I Q_i(t-1)  c_i + \frac{1}{2} \sum_{i=1}^I \l(\dotp{g_i^{t-1}}{\theta^t} - c_i \r)^2 \\ 
&\qquad\leq  \sum_{i=1}^I \dotp{Q_i(t-1)  g_i^{t-1}}{\theta^t}  - \sum_{i=1}^I Q_i(t-1)  c_i + 2 L^2, 
\end{aligned}
\end{align}
where we let $\|\mathbf{Q}(t)\|^2 = \sum_{i=1}^I Q_i^2(t)$ and $\|\mathbf{Q}(t-1)\|^2 = \sum_{i=1}^I Q_i^2(t-1)$, and the last inequality is due to 
\begin{align*}
\sum_{i=1}^I (\langle g_i^{t-1}, \theta^t \rangle - c_i )^2 \leq& 2 \sum_{i=1}^I [(\langle g_i^{t-1}, \theta^t \rangle)^2 + c_i^2]\\
\leq & 2 \sum_{i=1}^I [ \|g_i^{t-1}\|_\infty^2 \|\theta^t \|_1^2 + c_i^2]\\
\leq & 2 \sum_{i=1}^I [ L^2 \|g_i^{t-1}\|_\infty^2 + c_i^2] \\
\leq & 2  [ L^2 (\sum_{i=1}^I \|g_i^{t-1}\|_\infty)^2 + (\sum_{i=1}^I |c_i|)^2] \leq 4L^2
\end{align*}
by Assumption \ref{as:bounded} and the facts that $\sum_{s \in \mc S_k} \sum_{a \in \mc A} \sum_{ s' \in \mc S_{k+1}} \theta^t (s, a, s') = 1$ and $\theta^t (s, a, s') \geq 0$. Thus, summing up \eqref{eq:update_breg_tria} and \eqref{eq:Q_difference}, and then subtracting $\langle V f^{t-1}, \theta^{t-1} \rangle$ from both sides, we have
\begin{align*}
&V\big\langle f^{t-1},\theta^t - \theta^{t-1}\big\rangle + \frac{1}{2} \big( \|\mathbf{Q}(t)\|^2-\|\mathbf{Q}(t-1)\|^2 \big) + \alpha D(\theta^t, \widetilde{\theta}^{t-1}) \\
& \quad \leq  V\big\langle f^{t-1},\overline{\theta}^* - \theta^{t-1}\big\rangle + \sum_{i=1}^IQ_i(t-1)(\langle g_i^{t-1}, \overline{\theta}^*\rangle-c_i) + 
\alpha D(\overline{\theta}^*, \widetilde{\theta}^{t-1}) - \alpha D(\overline{\theta}^*, \theta^{t}) + 4 L^2 .
\end{align*}
We further need to show the lower bound of the term $V\big \langle f^{t-1}, \theta^t - \theta^{t-1}\big\rangle + \alpha D(\theta^t, \widetilde{\theta}^{t-1})$ on LHS of the above inequality. Specifically, we have
\begin{align*}
&V\big \langle f^{t-1}, \theta^t - \theta^{t-1}\big\rangle + \alpha D(\theta^t, \widetilde{\theta}^{t-1}) \\
&\qquad = V\big \langle f^{t-1},\theta^t - \widetilde{\theta}^{t-1}\big \rangle + V\big \langle f^{t-1}, \widetilde{\theta}^{t-1} - \theta^{t-1} \rangle + \alpha D(\theta^t, \widetilde{\theta}^{t-1}) \\
& \qquad \geq  -V\|f^{t-1}\|_\infty\cdot  \|\theta^t - \widetilde{\theta}^{t-1}\|_1 -   V  \|f^{t-1}\|_\infty \cdot  \| \widetilde{\theta}^{t-1} - \theta^{t-1}  \|_1 + \frac{\alpha}{2}  \sum_{k=0}^{L-1} \|\theta_k^t - \widetilde{\theta}_k^{t-1}\|_1^2  \\
& \qquad \geq  -V  \sum_{k=0}^{L-1} \|\theta_k^t - \widetilde{\theta}_k^{t-1}\|_1 - 2L\lambda  V   + \frac{\alpha}{2}  \sum_{k=0}^{L-1} \|\theta_k^t - \widetilde{\theta}_k^{t-1}\|_1^2 \geq -\frac{LV }{2\alpha}  - 2L\lambda  V ,  
\end{align*}
where the first inequality uses H\"older's inequality and Lemma \ref{lem:pins} that $D(\theta,\theta') = \sum_{k=1}^LD(\theta_k,\theta'_k)\geq 
\frac{1}{2} \sum_{k=1}^L\|\theta_k-\theta'_k\|_1^2$ with $\theta_k := [\theta(s, a, s')]_{s_k\in\mc S_k,a_k\in\mc A, s_{k+1}\in\mc S_{k+1}}$, the second inequality is due to $\widetilde{\theta}^{t-1}_k = (1-\lambda) \theta^{t-1}_k + \lambda \frac{\mf 1}{|\mathcal A| |\mathcal S_k| |\mathcal S_{k+1}| } $, the second inequality is due to  $\| \widetilde{\theta}^{t-1} - \theta^{t-1}  \|_1 = \sum_{k=0}^{L-1}\| \widetilde{\theta}^{t-1}_k - \theta^{t-1}_k  \|_1 =  \lambda \sum_{k=0}^{L-1} \big \| \theta^{t-1}_k - \frac{\mf 1}{|\mathcal A| |\mathcal S_k| |\mathcal S_{k+1}| } \big\|_1 \leq \lambda \sum_{k=0}^{L-1} \big  ( \big \| \theta^{t-1}_k \big \|_1 + \big \|\frac{\mf 1}{|\mathcal A| |\mathcal S_k| |\mathcal S_{k+1}| } \big\|_1 \big ) \leq 2\lambda L $, and the third inequality is by finding the minimal value of a quadratic function $-V  x + \frac{\alpha}{2} x^2$. Thus, we can show that with probability $1-\zeta$, the following inequality holds for all $t \leq  T+1$,
\begin{align}
&\frac{1}{2} \l( \|\mathbf{Q}(t)\|^2-\|\mathbf{Q}(t-1)\|^2 \r)  -\frac{LV }{2\alpha}  - 2L\lambda  V   \label{dpp:bound} \\
& \qquad \leq V \big\langle f^{t-1}, \overline{\theta}^* - \theta^{t-1} \big\rangle  + \sum_{i=1}^IQ_i(t-1)(\langle g_i^{t-1}, \overline{\theta}^*\rangle-c_i) + 
\alpha D(\overline{\theta}^*, \widetilde{\theta}^{t-1}) - \alpha D(\overline{\theta}^*, \theta^{t}) + 4 L^2. \nonumber
\end{align}
Note that according to Lemma \ref{lem:breg_diff}, we have 
\begin{align*}
D(\overline{\theta}^*, \widetilde{\theta}^{t-1}) -  D(\overline{\theta}^*, \theta^{t}) &= D(\overline{\theta}^*, \widetilde{\theta}^{t-1}) - D(\overline{\theta}^*, \theta^{t-1})+D(\overline{\theta}^*, \theta^{t-1})  -  D(\overline{\theta}^*, \theta^{t}) \\
&\leq \lambda L \log |\mathcal{S}|^2 |\mathcal{A}| +D(\overline{\theta}^*, \theta^{t-1})  -  D(\overline{\theta}^*, \theta^{t}).
\end{align*} 
Therefore, plugging the above inequality into \eqref{dpp:bound} and rearranging the terms, we further get
\begin{align*}
 V\dotp{f^{t-1}}{\theta^{t-1} - \overline{\theta}^*}  &\leq \frac{1}{2} \l( \|\mathbf{Q}(t-1)\|^2-\|\mathbf{Q}(t)\|^2 \r)   + \sum_{i=1}^IQ_i(t-1)(\langle g_i^{t-1}, \overline{\theta}^*\rangle-c_i)  \\
&\quad + \alpha \lambda L \log |\mathcal{S}|^2 |\mathcal{A}| + \alpha D(\overline{\theta}^*, \theta^{t-1})  -  \alpha D(\overline{\theta}^*, \theta^{t}) + 4 L^2 +  \frac{LV }{2\alpha}  + 2L\lambda  V  . 
\end{align*}
Thus, by taking summation on both sides of the above inequality from $2$ to $T+1$ and assuming $\mathbf{Q}(0) = 0$, we would obtain that with probability at least $1- \zeta$, the following inequality holds
\begin{align}
\begin{aligned} \label{eq:regret_except_Q_prod}
 \sum_{t=2}^{T+1} \dotp{f^{t-1}}{\theta^{t-1} - \overline{\theta}^*}  &\leq \frac{1}{V}\sum_{t=2}^{T+1}\sum_{i=1}^IQ_i(t-1)(\langle g_i^{t-1}, \overline{\theta}^*\rangle-c_i)   + \frac{T\alpha \lambda L \log |\mathcal{S}|^2 |\mathcal{A}|}{V} \\
&\qquad  + \frac{\alpha D(\overline{\theta}^*, \theta^{0}) + 4 K L^2  }{V}  +  \frac{LT }{2\alpha}  + 2L\lambda  T  . 
\end{aligned}
\end{align} 
It is not difficult to compute that $D(\overline{\theta}^*, \theta^{0}) \leq L \log|\mathcal S|^2 |\mathcal A|$ according to the initialization of $\theta^{1}$ by the uniform distribution. Then, by rearranging the terms and shifting the index, we rewrite \eqref{eq:regret_except_Q_prod} as
\begin{align*}
& \sum_{t=1}^{T} \dotp{f^t}{\theta^t - \overline{\theta}^*}  \leq \frac{1}{V}\sum_{t=1}^T\sum_{i=1}^IQ_i(t)(\langle g_i^t, \overline{\theta}^*\rangle-c_i)   + \frac{ 4 T L^2 +(\lambda T+1)\alpha  L \log |\mathcal{S}|^2 |\mathcal{A}|}{V} +  \frac{LT }{2\alpha}  + 2L\lambda  T  . 
\end{align*}
This completes the proof.
\end{proof}

\subsection{Proof of Lemma \ref{lem:Q_drift_diff}}

\begin{lemma}[Lemma 5 of \citet{yu2017online}] \label{lem:drift} Let $\{Z(t), t \geq 0\}$ be a discrete time stochastic process adapted to a filtration  $\{ \mathcal{U}^t, t \geq 0 \}$ with $Z(0) = 0$ and $\mathcal{U}^0 = \{ \emptyset, \Omega \}$. Suppose there exists an integer $\tau  > 0$, real constants $\theta > 0$, $\rho_{\max} > 0$, and $0< \kappa \leq \rho_{\max}$ such that 
\begin{align*}
|Z(t+1)-Z(t)| &\leq \rho_{\max},\\
\E[ Z(t+\tau ) - Z(t) | \mathcal{U}^t ] &\leq\left\{\begin{matrix}
\tau  \rho_{\max},\ \text{ if } Z(t) < \psi\\ 
-\tau \kappa, \quad \text{ if } Z(t) \geq \psi
\end{matrix}\right.
\end{align*}
hold for all $t\in \{1,2,...\}$. Then for any constant $0<\delta<1$, 
with probability at least $1-\delta$, we have 
\begin{align*}
Z(t) \leq \psi + \tau  \frac{4\rho_{\max}^2}{\kappa} \log\bigg(1+\frac{8\rho_{\max}^2}{\kappa^2} e^{\kappa / (4\rho_{\max})}\bigg) + \tau  \frac{4\rho_{\max}^2}{\kappa} \log\frac{1}{\delta}, ~~\forall t \in \{1,2,...\}.
\end{align*}
\end{lemma}

Now, we are in position to give the proof of Lemma \ref{lem:Q_drift_diff}.
%

\begin{proof} [Proof of Lemma \ref{lem:Q_drift_diff}] The proof of this Lemma is based on applying the lemma \ref{lem:drift} to our problem. Thus, this proof mainly focuses on showing that the variable $\|\mathbf{Q}(t)\|_2$ satisfies the condition of Lemma \ref{lem:drift}.  By the updating rule of $Q_i(t)$, i.e., $Q_i(t+1) = \max\{ Q_i(t) + \langle g_i^t, \theta^{t+1} \rangle  - c_i, 0\}$, we have
\begin{align*}
\l|\|\mathbf{Q}(t+1)\|_2-\|\mathbf{Q}(t)\|_2 \r| \leq&  \|\mathbf{Q}(t+1) - \mathbf{Q}(t)\|_2  \\
=&  \sqrt{\sum_{i=1}^I |Q_i(t+1) - Q_i(t)|^2} \\
\leq &  \sqrt{\sum_{i=1}^I |\langle g_i^t, \theta^{t+1} \rangle  - c_i|^2} ,
\end{align*}
where the first inequality is due to triangle inequality, and the second inequality is by the fact that $|\max\{ a+ b, 0\} - a| \leq |b|$ if $a\geq 0$. Then, by Assumption \ref{as:bounded}, we further have
\begin{align*}
\sqrt{\sum_{i=1}^I |\langle g_i^t, \theta^{t+1} \rangle  - c_i|^2} \leq \sum_{i=1}^I |\langle g_i^t, \theta^{t+1} \rangle  - c_i|\leq \sum_{i=1}^I (\| g_i^t \|_\infty \|\theta^{t+1} \|_1  + |c_i|) \leq 2L, 
\end{align*}
which therefore implies
\begin{align}
\l|\|\mathbf{Q}(t+1)\|_2-\|\mathbf{Q}(t)\|_2 \r| \leq 2L. \label{eq:Q_shift}
\end{align}
Thus, with the above inequality, we have
\begin{align}
\begin{aligned}\label{eq:Q_diff_hard_bound}
\|\mathbf{Q}(t+\tau )\|_2 - \|\mathbf{Q}(t)\|_2  &\leq |\|\mathbf{Q}(t+\tau )\|_2 - \|\mathbf{Q}(t)\|_2 | \\
&\leq \sum_{\tau = 1}^{\tau } \l|\|\mathbf{Q}(t+\tau)\|_2 - \|\mathbf{Q}(t+\tau-1)\|_2 \r|  \leq  2  \tau  L, 
\end{aligned}
\end{align}
such that 
\begin{align}
\E [\|\mathbf{Q}(t+\tau )\|_2 - \|\mathbf{Q}(t)\|_2 | \mathcal{F}^{t-1} ]   \leq  2 \tau  L.\label{eq:exp_Q_diff_hard_bound}
\end{align}
where $\mathcal{F}^{t-1}$ represents the system randomness up to the $(t-1)$-th episode and $\mathbf{Q}(t)$ depends on $\mathcal{F}^{t-1}$ according to its updating rule. 
Note that \eqref{eq:Q_diff_hard_bound} in fact indicates that the random process $\|\mathbf{Q}(t+\tau )\|_2 - \|\mathbf{Q}(t)\|_2$ is bounded by the value $2 \tau  L$.

Next, we need to show that there exist $\psi$ and $\kappa$ such that $\E [\|\mathbf{Q}(t+\tau )\|_2 - \|\mathbf{Q}(t)\|_2 | \mathcal{F}^{t-1} ]  \leq - \tau \kappa$ if $ \|\mathbf{Q}(t)\|_2 \geq \psi$. Recall the definition of the event $\mathcal{D}_T$ in \eqref{eq:event_opt_feas}. Therefore, we have that with probability at least $1- \zeta $, the event $\mathcal{D}_T$ happens, such that for all $t' = 2,...,T+1$ and any $\theta \in \cap_{\ell=1}^{\ell(T+1)} \Delta(\ell, \zeta)$, the following holds
\begin{align*}
V\big\langle f^{t'-1}, \theta^{t'-1} - \overline{\theta}^*\big\rangle  &\leq \frac{1}{2} \l( \|\mathbf{Q}(t'-1)\|_2^2-\|\mathbf{Q}(t')\|_2^2 \r)   + \sum_{i=1}^IQ_i(t'-1)(\langle g_i^{t'-1}, \theta\rangle-c_i)  \\
&\quad  + \alpha \lambda L \log |\mathcal{S}|^2 |\mathcal{A}| + \alpha D(\theta,\widetilde{\theta}^{t'-1})  -  \alpha D(\theta, \theta^{t'}) + 4  L^2 +  \frac{LV }{2\alpha}  + 2L\lambda  V  ,
\end{align*}
which adopts similar proof techniques to \eqref{dpp:bound}. Then, the above inequality further leads to the following inequality by rearranging the terms  
\begin{align*}
 \|\mathbf{Q}(t')\|_2^2 -\|\mathbf{Q}(t'-1)\|_2^2  &\leq  -2V\dotp{f^{t'-1}}{\theta^{t'-1} - \theta}  +  2 \sum_{i=1}^IQ_i(t'-1)(\langle g_i^{t'-1}, \theta\rangle-c_i)  \\
&\quad   + 2\alpha \lambda L \log |\mathcal{S}|^2 |\mathcal{A}| + 2\alpha D(\theta, \widetilde{\theta}^{t'-1})  - 2 \alpha D(\theta, \theta^{t'}) + 8L^2 +  \frac{LV }{\alpha}  + 4L\lambda  V.
\end{align*}
Taking summation from $t + 1$ to $\tau  + t$ on both sides of the above inequality, the following inequality holds with probability $1-\zeta$ for any $\tau>0$ and $t$ satisfying $1 \leq k \leq K+1-\tau$
\begin{align}
\begin{aligned}\label{eq:Q_diff_high_prob_bound}
 &\hspace*{-0.3cm}\|\mathbf{Q}(\tau  + t)\|_2^2 -\|\mathbf{Q}(t)\|_2^2  \\
 &\hspace*{-0.3cm}\quad \leq -2V \sum_{t' = t+1 }^{\tau +t} \dotp{f^{t'-1}}{\theta^{t'-1} - \theta} +  2  \sum_{t' = t+1 }^{\tau +t} \sum_{i=1}^IQ_i(t'-1)(\langle g_i^{t'-1}, \theta\rangle-c_i)  + 2\alpha D(\theta, \widetilde{ \theta}^{t}) \\
&\hspace*{-0.3cm} \qquad - 2 \alpha D(\theta, \widetilde{\theta}^{\tau  + t }) + \sum_{t'=t+1}^{\tau  + t}  2 \alpha [ D(\theta, \widetilde{\theta}^{t'-1})  - D(\theta, \theta^{t'-1}) ] + 8\tau L^2  +  \frac{\tau LV }{\alpha}  + 4\tau L\lambda  V  . 
\end{aligned}
\end{align}
Particularly, in \eqref{eq:Q_diff_high_prob_bound}, the term $-2\alpha D(\theta, \theta^{t'-1}) \leq 0$ due to the non-negativity of unnormalized KL divergence. By Lemma \ref{lem:breg_diff}, we can bound 
\begin{align*}
\sum_{\tau=t+1}^{\tau  + t}  2 \alpha [ D(\theta, \widetilde{\theta}^{t'-1})  - D(\theta, \theta^{t'-1}) ] \leq 2\alpha  \tau  L \log |\mathcal S|^2 |\mathcal A|. 
\end{align*}
For the term $2\alpha D(\theta, \widetilde{ \theta}^{t})$, by Lemma \ref{lem:breg_diff}, we can bound it as
\begin{align*}
2\alpha D(\theta, \widetilde{ \theta}^{t})  \leq  2\alpha   L \log ( |\mathcal{S}|^2 |\mathcal{A}|/\lambda).
\end{align*}
Moreover, we can decompose the term $2V\sum_{t'=t+1}^{\tau  + t}\langle f^{t'-1}, \theta- \theta^{t'-1}\rangle +  2 \sum_{t'=t+1}^{\tau  + t} \sum_{i=1}^IQ_i(t'-1)(\langle g_i^{t'-1}, \overline{\theta}^*\rangle-c_i)$ in \eqref{eq:Q_diff_high_prob_bound} as 
\begin{align*}
&2V \sum_{t'=t+1}^{\tau  + t}\dotp{f^{t'-1}}{ \theta-\theta^{t'-1} }     +  2 \sum_{t'=t+1}^{\tau  + t}  \sum_{i=1}^IQ_i(t'-1)(\langle g_i^{t'-1}, \theta\rangle-c_i) \\
&\qquad =  2V \sum_{t'=t+1}^{\tau  + t}\dotp{f^{t'-1}}{ \theta -\theta^{t'-1} }     +  2 \sum_{i=1}^IQ_i(t) \sum_{t'=t+1}^{\tau  + t}   (\langle g_i^{t'-1}, \theta \rangle-c_i) \\
&\qquad\qquad \quad  +  2 \sum_{t'=t+2}^{\tau  + t}   \sum_{i=1}^I [ Q_i(t'-1) - Q_i(t)] (\langle g_i^{t'-1}, \theta \rangle-c_i) \\
&\qquad \leq   2V \sum_{t'=t+1}^{\tau  + t}\dotp{f^{t'-1}}{ \theta}     +  2 \sum_{i=1}^IQ_i(t) \sum_{t'=t+1}^{\tau  + t}   (\langle g_i^{t'-1},\theta\rangle-c_i) +  2L\tau ^2 + 2VL\tau ,
\end{align*}
where the last inequality is due to 
\begin{align*}
- 2V \sum_{t'=t+1}^{\tau  + t}\dotp{f^{t'-1}}{\theta^{t'-1} } \leq 2V \sum_{t'=t+1}^{\tau  + t} \sum_{k=0}^{L-1} \sum_{s \in \mc S_k} \sum_{a \in \mc A} \sum_{s' \in \mc S_{k+1}}  f^{t'-1}(s, a, s')\theta^{t'-1} (s, a, s') \leq 2VL\tau ,
\end{align*}
as well as
\begin{align*}
&2 \sum_{t'=t+2}^{\tau  + t}   \sum_{i=1}^I [ Q_i(t'-1) - Q_i(t)] (\langle g_i^{t'-1}, \theta \rangle-c_i) \\
&\qquad \leq 2 \sum_{t'=t+2}^{\tau  + t}   \sum_{i=1}^I  \sum_{r=t}^{t'-2}|\langle g_i^r, \theta^{r+1} \rangle - c_i| \cdot |\langle g_i^{t'-1}, \theta\rangle-c_i |  \\
&\qquad \leq   \sum_{t'=t+2}^{\tau  + t}   \sum_{r=t}^{t'-2} \sqrt{ \sum_{i=1}^I |\langle g_i^r, \theta^{r+1} \rangle - c_i|^2}   + \sum_{t'=t+2}^{\tau  + t}  \sum_{r=t}^{t'-2}  \sqrt{\sum_{i=1}^I   |\langle g_i^{t'-1}, \theta\rangle-c_i |^2}  \\
&\qquad \leq  2 L \tau ^2,
\end{align*}
by $Q_i(t+1) = \max\{ Q_i(t) + \langle g_i^t, \theta^{t+1} \rangle  - c_i, 0\}$ and $|\max\{ a+ b, 0\} - a| \leq |b|$ if $a\geq 0$ for the first inequality, and Assumption \ref{as:bounded} for the last inequality. 

Therefore, taking conditional expectation on both sides of \eqref{eq:Q_diff_high_prob_bound} and combining the above upper bounds for certain terms in \eqref{eq:Q_diff_high_prob_bound}, we can obtain for any $\theta\in \cap_{\ell=1}^{\ell(T+1)} \Delta(\ell, \zeta)$, 
\begin{align}
\begin{aligned}\label{eq:exp_Q_diff_high_prob_bound}
&\E[\|\mathbf{Q}(\tau  + t)\|^2 -\|\mathbf{Q}(t)\|^2 | \mathcal{F}^{t-1}, \mathcal{D}_T ]   \\
& \qquad \leq    2 \tau ^2 L + 2\alpha   L \log ( |\mathcal{S}|^2 |\mathcal{A}|/\lambda)   \\
&\qquad \quad +  2V\tau \E \bigg[ \frac{1}{\tau }\sum_{t'=t+1}^{\tau  + t}\langle f^{t'-1}, \theta \rangle     +  \frac{1}{\tau } \sum_{i=1}^I \frac{Q_i(t)}{V} \sum_{t'=t+1}^{\tau  + t}   (\langle g_i^{t'-1}, \theta \rangle-c_i)  \bigg| \mathcal{F}^{t-1}, \mathcal{D}_T \bigg]     \\
&\qquad \quad + 2\alpha \lambda \tau  L \log |\mathcal S|^2 |\mathcal A|   + 8\tau L^2  +  \frac{\tau LV }{\alpha} + 4\tau L\lambda  V   + 2VL\tau .  
\end{aligned}
\end{align}
Thus, it remains to bound the term $\E[ \frac{1}{\tau }\sum_{t'=t+1}^{\tau  + t}\langle f^{t'-1}, \theta \rangle     +  \frac{1}{\tau } \sum_{i=1}^I \frac{Q_i(t)}{V} \sum_{t'=t+1}^{\tau  + t}   (\langle g_i^{t'-1}, \theta\rangle-c_i)  | \mathcal{F}^{t-1},\mathcal{D}_T ]$ so as to give an upper bound of the right-hand side of \eqref{eq:exp_Q_diff_high_prob_bound}. Given the event $\mc D_T$ happens such that $\Delta \subseteq \cap_{\ell=1}^{\ell(T+1)} \Delta(\ell, \zeta) \neq \emptyset$, and since $\theta$ is any  vector in the set $\cap_{\ell=1}^{\ell(T+1)} \Delta(\ell, \zeta)$, we can give an upper bound of \eqref{eq:exp_Q_diff_high_prob_bound} by bounding a term $q^{(t,\tau )}\Big( \frac{\mathbf{Q}(t)}{V} \Big)$, which is by
\begin{align*}
&\min_{\theta \in \cap_{\ell=1}^{\ell(T)} \Delta(\ell, \zeta)} \E \Big[ \frac{1}{\tau }\sum_{t'=t+1}^{\tau  + t} \big \langle f^{t'-1},  \theta \big \rangle +  \frac{1}{\tau } \sum_{i=1}^I \frac{Q_i(t)}{V} \sum_{t'=t+1}^{\tau  + t}   (\langle g_i^{t'-1}, \theta\rangle-c_i)  \Big| \mathcal{F}^t, \mathcal{D}_T \Big] \\
&\quad =  \min_{\theta \in \cap_{\ell=1}^{\ell(T)} \Delta(\ell, \zeta)}  \big\langle f^{(t,\tau )}, \theta \big\rangle  +  \sum_{i=1}^I \frac{Q_i(t)}{V}  (\langle g_i, \theta\rangle-c_i) \\
&\quad \leq  \min_{\theta \in \Delta}  \big\langle f^{(t,\tau )}, \theta \big\rangle  +  \sum_{i=1}^I \frac{Q_i(t)}{V}  (\langle g_i, \theta\rangle-c_i)  = q^{(t,\tau )}\Big( \frac{\mathbf{Q}(t)}{V} \Big),
\end{align*}
where the inequality is due to $\Delta \subseteq  \cap_{\ell=1}^{\ell(T+1)} \Delta(\ell, \zeta)$ given $\mc D_T$ happens and the last equality is obtained according to the definition of the dual function $q$ in Section \ref{sec:main}. We can bound $q^{(t,\tau )}\big( \frac{\mathbf{Q}(t)}{V} \big)$ in the following way. 

According to Assumption \ref{as:selm}, we assume that one dual solution is $\eta^*_{t,\tau} \in \mathcal{V}_{t,\tau }^*$. We let $\overline{\vartheta}$ be the maximum of all $\vartheta$ and $\overline{\sigma}$ be the minimum of all $\sigma$. Thus, when $\dist(\frac{\mathbf{Q}(t)}{V}, \mathcal{V}_{t,\tau }^*) \geq \overline{\vartheta}$, we have
\begin{align*}
q^{(t,\tau )}\Big(\frac{\mathbf{Q}(t)}{V}\Big) = & q^{(t,\tau )}\Big(\frac{\mathbf{Q}(t)}{V}\Big) - q^{(t,\tau )}(\eta^*_{t,\tau }) + q^{(t,\tau )}(\eta^*_{t,\tau }) \\
\leq & -\overline{\sigma}\Big\|\eta^*_{t,\tau } - \frac{\mathbf{Q}(t)}{V} \Big\|_2 + \big\langle f^{(t,\tau )}, \theta^*_{t,\tau} \big\rangle\\ 
\leq & -\overline{\sigma}\Big\| \frac{\mathbf{Q}(t)}{V} \Big\|_2 + \overline{\sigma}\|\eta^*_{t,\tau } \|_2 + \sum_{k=0}^{L-1} \sum_{s \in \mc S_k} \sum_{a \in \mc A} \sum_{s' \in \mc S_{k+1}}  f^{(t,\tau )}(s,a, s') \theta^*_{t,\tau}(s,a, s')\\
\leq & -\overline{\sigma}\Big\| \frac{\mathbf{Q}(t)}{V} \Big\|_2 + \overline{\sigma} B + L ,
\end{align*}
where the first inequality is due to the weak error bound in Lemma \ref{lem:leb} and weak duality with $\theta^*_{t,\tau}$ being one primal solution, the second inequality is by triangle inequality, and the third inequality is by Assumption \ref{as:bounded} and Assumption \ref{as:selm}. On the other hand, when $\dist(\frac{\mathbf{Q}(t)}{V}, \mathcal{V}_{t,\tau }^*) \leq \overline{\vartheta}$, we have
\begin{align*}
q^{(t,\tau )}\Big(\frac{\mathbf{Q}(t)}{V}\Big) = & \min_{\theta \in \Delta  }  \big \langle f^{(t,\tau )},\theta\big \rangle +  \sum_{i=1}^I \frac{Q_i(t)}{V}  (\langle g_i, \theta\rangle-c_i) \\
= & \min_{\theta \in \Delta  }  \big \langle f^{(t,\tau )}, \theta \big\rangle  +  \sum_{i=1}^I [\eta^*_{t,\tau }]_i  (\langle g_i, \theta\rangle-c_i) +  \sum_{i=1}^I \Big(\frac{Q_i(t)}{V}-[\eta^*_{t,\tau }]_i\Big)  (\langle g_i, \theta\rangle-c_i) \\
\leq & q^{(t,\tau )}(\eta^*_{t,\tau })  +   \Big\|\frac{\mathbf{Q}(t)}{V}-\eta^*_{t,\tau }\Big\|_2 \|\mf g( \theta)-\mf c\|_2\\
\leq & L + 2 \overline{\vartheta} L ,
\end{align*}
where the first inequality is by the definition of $q^{(t,\tau )}(\eta^*_{t,\tau })$ and Cauchy-Schwarz inequality,  and the second inequality is due to weak duality and Assumption \ref{as:bounded} such that
\begin{align*}
&q^{(t,\tau )}(\eta^*_{t,\tau }) \leq \big\langle f^{(t,\tau )}, \theta^*_{t,\tau} \big\rangle \leq \big\|f^{(t,\tau )}\big\|_\infty \| \theta^*_{t,\tau}\|_1   \leq  L,\\
&\Big\|\frac{\mathbf{Q}(t)}{V}-\eta^*_{t,\tau }\Big\|_2 \|\mf g( \theta)-\mf c\|_2 \leq \overline{\vartheta}  \sqrt{\sum_{i=1}^I \Big|\langle  g_i, \theta\rangle  -c_i \Big|^2 } \leq \overline{\vartheta}  \sum_{i=1}^I ( \| g_i\|_\infty  \|\theta\|_1 + |c_i| ) \leq  2\overline{\vartheta} L.
\end{align*}
Now we can combine the two cases as follows
\begin{align}
q^{(t,\tau )}\Big( \frac{\mathbf{Q}(t)}{V} \Big) \leq  -\overline{\sigma}\Big\| \frac{\mathbf{Q}(t)}{V} \Big\|_2 + \overline{\sigma} B + 2L + 2 \overline{\vartheta}L + \overline{\sigma} \overline{\vartheta}. \label{eq:bound_q_func}
\end{align}
The bound in \eqref{eq:bound_q_func} is derived by the following discussion:
\begin{itemize}
\item[\textbf{(1})] When $\dist\big( \frac{\mathbf{Q}(t)}{V}, \mathcal{V}^*_{t,\tau } \big) \geq \overline{\vartheta}$, we have 
$$q^{(t,\tau )}\big( \frac{\mathbf{Q}(t)}{V} \big) \leq  -\overline{\sigma}\big\| \frac{\mathbf{Q}(t)}{V} \big\|_2 + \overline{\sigma} B + L  \leq  -\overline{\sigma}\big\| \frac{\mathbf{Q}(t)}{V} \big\|_2 + \overline{\sigma} B + 2L+ 2\overline{\vartheta}L + \overline{\sigma} \overline{\vartheta}.$$

\item[\textbf{(2})] When $\dist\big( \frac{\mathbf{Q}(t)}{V}, \mathcal{V}^*_{t,\tau } \big) < \overline{\vartheta}$, we have 
$$q^{(t,\tau )}\big( \frac{\mathbf{Q}(t)}{V} \big) \leq  L  + 2\overline{\vartheta} L   \leq  -\overline{\sigma}\big\| \frac{\mathbf{Q}(t)}{V} \big\|_2 + \overline{\sigma} B + 2L + 2\overline{\vartheta}L + \overline{\sigma} \overline{\vartheta},$$ since $-\overline{\sigma}\big\| \frac{\mathbf{Q}(t)}{V}\big\|_2 + \overline{\sigma} \overline{\vartheta} + \overline{\sigma} B \geq -\overline{\sigma} \cdot \dist \big( \frac{\mathbf{Q}(t)}{V}, \mathcal{V}^*_{t,\tau }\big) + \overline{\sigma} \overline{\vartheta} + \overline{\sigma} B - \overline{\sigma}B = \overline{\sigma} \big[ -\dist \big( \frac{\mathbf{Q}(t)}{V}, \mathcal{V}^*_{t,\tau }\big) + \overline{\vartheta} \big]   \geq 0$.
\end{itemize}

Therefore, plugging \eqref{eq:bound_q_func} into \eqref{eq:exp_Q_diff_high_prob_bound}, we can obtain that given the event $\mathcal{D}_T$ happens, the following holds 
\begin{align}
\begin{aligned}\label{eq:bound_q_func_diff}
&\E[\|\mathbf{Q}(\tau  + t)\|_2^2 -\|\mathbf{Q}(t)\|_2^2 | \mathcal{F}^t, \mathcal{D}_T ]\\
&\qquad \leq 2\tau ^2 L  + \tau  C_{V, \alpha, \lambda}  + 2\alpha   L \log ( |\mathcal{S}|^2 |\mathcal{A}|/\lambda)
-  2\tau  \overline{\sigma}\| \mathbf{Q}(t)\|_2, 
\end{aligned}
\end{align}
where we define 
\begin{align*}
C_{V, \alpha, \lambda} :=  2(\overline{\sigma} B+  \overline{\sigma} ~ \overline{\vartheta})V + (6  + 4\overline{\vartheta})  V L    +  \frac{VL }{\alpha} + 4L\lambda  V  + 2\alpha \lambda L \log |\mathcal S|^2 |\mathcal A|+ 8 L^2.
\end{align*}
We can see that if $\|\mathbf{Q}(t)\|_2 \geq (2 \tau  L  + 
C_{V,\alpha, \lambda})/\overline{\sigma} + 2\alpha \lambda L \log (|\mathcal{S}|^2 |\mathcal A| / \lambda) /(\overline{\sigma}\tau ) + \tau  \overline{\sigma} / 2$, then according to \eqref{eq:bound_q_func_diff}, there is
\begin{align*}
\E[\|\mathbf{Q}(\tau  + t)\|^2 | \mathcal{F}^{t-1}, \mathcal{D}_T ] \leq & \|\mathbf{Q}(t)\|^2 -  \tau  \overline{\sigma}\| \mathbf{Q}(t)\|_2 - \frac{\overline{\sigma}^2 \tau ^2}{2} \\
\leq & \|\mathbf{Q}(t)\|_2^2 -  \tau  \overline{\sigma}\| \mathbf{Q}(t)\|_2 + \frac{\overline{\sigma}^2 \tau ^2}{4}\\
\leq & \Big(\|\mathbf{Q}(t)\|_2 -  \frac{\tau  \overline{\sigma}}{2}\Big)^2.
\end{align*}
Due to $\|\mathbf{Q}(t)\|_2 \geq \frac{\tau  \overline{\sigma}}{2}$ and by Jensen's inequality, we have
\begin{align}
\E[\|\mathbf{Q}(\tau  + t)\|_2 | \mathcal{F}^{t-1}, \mathcal{D}_T ] \leq  \sqrt{\E[\|\mathbf{Q}(\tau  + t)\|_2^2 | \mathcal{F}^{t-1}, \mathcal{D}_T ]} \leq  \|\mathbf{Q}(t)\|_2 -  \frac{\tau  \overline{\sigma}}{2}. \label{eq:final_exp_Q_diff_high_prob_bound}
\end{align}
Then we can compute the expectation $\E[\|\mathbf{Q}(\tau  + t)\|_2^2 -\|\mathbf{Q}(t)\|_2^2 | \mathcal{F}^{t-1}]$ according to the law of total expectation. With \eqref{eq:Q_diff_hard_bound} and \eqref{eq:final_exp_Q_diff_high_prob_bound}, we can obtain that
\begin{align*}
&\E[\|\mathbf{Q}(\tau  + t)\|_2 -\|\mathbf{Q}(t)\|_2 | \mathcal{F}^{t-1}] \\
&\quad = P(\mathcal{D}_T) \E[\|\mathbf{Q}(\tau  + t)\|_2 -\|\mathbf{Q}(t)\|_2 | \mathcal{F}^{t-1}, \mathcal{D}_T] +P(\overline{\mathcal{D}}_T) \E[\|\mathbf{Q}(\tau  + t)\|_2 -\|\mathbf{Q}(t)\|_2 | \mathcal{F}^{t-1}, \overline{\mathcal{D}}_T]\\
&\quad  \leq -\frac{\tau  \overline{\sigma}}{2}(1-  \zeta )  + 2 \zeta  \tau  L = - \tau  \Big[ \frac{\overline{\sigma}}{2} -  \zeta \Big(\frac{\overline{\sigma}}{2} + 2L \Big) \Big] \leq - \frac{\overline{\sigma}}{4}\tau  ,
\end{align*}
where we let $\overline{\sigma}/4 \geq  \zeta (\overline{\sigma}/2 + 2L)$.

Summarizing the above results, we know that if $\overline{\sigma}/4 \geq \zeta  (\overline{\sigma}/2 + 2L)$, then
\begin{align*}
\l|\|\mathbf{Q}(t+1)\|_2-\|\mathbf{Q}(t)\|_2 \r| &\leq 2 L,\\
\E [\|\mathbf{Q}(t+\tau )\|_2 - \|\mathbf{Q}(t)\|_2 | \mathcal{F}^t ] &\leq\left\{\begin{matrix}
2 \tau  L , \qquad \text{ if } \|\mathbf{Q}(t)\|_2 < \psi\\ 
-\frac{\overline{\sigma}}{4}\tau   , ~\quad \ \ \text{ if } \|\mathbf{Q}(t)\|_2 \geq \psi
\end{matrix}\right.,
\end{align*}
where we let
\begin{align*}
&\psi=\frac{ 2 \tau  L  + C_{V,\alpha, \lambda}}{\overline{\sigma}} + \frac{2\alpha  L \log ( |\mathcal{S}|^2 |\mathcal A|/ \lambda) }{\overline{\sigma}\tau } + \frac{\tau  \overline{\sigma}}{2}, \\
&C_{V, \alpha, \lambda} =  2(\overline{\sigma} B+  \overline{\sigma} ~ \overline{\vartheta})V + (6  + 4\overline{\vartheta})  V L    +  \frac{VL }{\alpha} + 4L\lambda  V  + 2\alpha \lambda L \log |\mathcal S|^2 |\mathcal A|+ 8 L^2.
\end{align*}
Directly by Lemma \ref{lem:drift}, for a certain $t \in [T+1-\tau]$, the following inequality holds with probability at least $1-\delta$, 
\begin{align}
\|\mathbf{Q}(t)\|_2 \leq& \psi + \tau  \frac{512 L^2 }{ \overline{\sigma}} \log \bigg( 1+\frac{128L^2}{\overline{\sigma}^2} e^{\overline{\sigma} / (32L )} \bigg) + \tau  \frac{64 L^2 }{\overline{\sigma}} \log \frac{1}{\delta}. \label{eq:Q_bound}
\end{align}
Further employing union bound for probabilities, we have that with probability at least $1-(T+1-\tau)\delta \geq 1-T\delta$, for any $t\in [T+1-\tau]$, the above inequality \eqref{eq:Q_bound} holds. Note that \eqref{eq:Q_bound} only holds when $t\in [T+1-\tau]$. For $T+2-\tau\leq t \leq T+1$, when \eqref{eq:Q_bound} holds for $t\in [T+1-\tau]$, combining \eqref{eq:Q_bound} and \eqref{eq:Q_shift}, we have
\begin{align}
    \|\mathbf{Q}(t)\|_2 \leq& \psi + \tau  \frac{512 H^2 }{ \overline{\sigma}} \log \bigg( 1+\frac{128H^2}{\overline{\sigma}^2} e^{\overline{\sigma} / (32H )} \bigg) + \tau  \frac{64 H^2 }{\overline{\sigma}} \log \frac{1}{\delta} + 2\tau L. \label{eq:Q_bound_final}
\end{align}
Thus, with probability at least $1-T\delta$, for any k satisfying $1\leq t\leq T+1$, the inequality \eqref{eq:Q_bound_final} holds.
We can understand the upper bound of the term $\log\big (1+\frac{128 L^2}{\overline{\sigma}^2} e^{\overline{\sigma} / (32 L )}\big)$ in the following way: \textbf{(1)} if $\frac{128 L^2 }{\overline{\sigma}^2} e^{\overline{\sigma} / (32L )} \geq 1$, then this term is bounded by $\log \big(\frac{256L^2}{\overline{\sigma}^2} e^{\overline{\sigma} / (32 L )} \big) =  \frac{\overline{\sigma}}{32L} + \log\frac{256L^2 }{\overline{\sigma}^2}$; \textbf{(2)} if $\frac{128 L^2}{\overline{\sigma}^2} e^{\overline{\sigma} / (32L)} < 1$, then the term is bounded by $\log 2$. Thus, we have
\begin{align*}
\log \bigg(1+\frac{128L^2 }{\overline{\sigma}^2} e^{\overline{\sigma} / (32L)} \bigg ) \leq \log 2 + \frac{\overline{\sigma}}{32 L} + \log\frac{256L^2 }{\overline{\sigma}^2}.
\end{align*}
This discussion shows that the $\log$ term in \eqref{eq:Q_bound} will not introduce extra dependency on $L$ except a $\log L$ term. This completes our proof.
\end{proof}

\subsection{Proof of Lemma \ref{lem:term_II}}

\begin{lemma}[Lemma 9 of \citet{yu2017online}] \label{lem:supermar}
Let $\{Z(t), t \geq 0\}$ be a supermartingale adapted to a filtration $\{\mathcal{U}^t, t\geq 0\}$ with $Z(0) = 0$ and $\mathcal{U}^0 = \{\emptyset, \Omega\}$, i.e., $\E[Z(t+1) | \mathcal{U}^t] \leq Z(t)$, $\forall t \geq 	0$. Suppose there exists a constant $\varsigma  > 0$ such that $\{ |Z(t+1) - Z(t)|> \varsigma \} \subset \{Y(t) > 0\}$, where $Y(t)$ is process with $Y(t)$ adpated to $\mathcal{U}^t$ for all $t \geq 0$. Then, for all $z > 0$, we have
\begin{align*}
\Pr(Z(t) \geq z ) \leq e^{-z^2/(2t\varsigma^2)} + \sum_{\tau=0}^{t-1} \Pr(Y(\tau) > 0), \forall t \geq 1.
\end{align*}
\end{lemma}

We are in position to give the proof of Lemma \ref{lem:term_II}.
%

\begin{proof} [Proof of Lemma \ref{lem:term_II}]

Now we compute the upper bound of the term  $\sum_{t=1}^T\sum_{i=1}^IQ_i(t)(\langle g_i^{t}, \overline{\theta}^*\rangle-c_i)$. Note that $Z(t):=\sum_{\tau=1}^t\sum_{i=1}^IQ_i(\tau)(\langle g_i^{\tau}, \overline{\theta}^*\rangle-c_i)$ is  supermartigale which is verified by 
\begin{align*}
\E [Z(t)| \mathcal{F}^{t-1}]=&\E \Big[\sum_{\tau=1}^t\sum_{i=1}^IQ_i(\tau)(\langle g_i^{\tau}, \overline{\theta}^*\rangle-c_i) \Big| \mathcal{F}^{t-1} \Big] \\
 =& \sum_{i=1}^I \E [Q_i(t)| \mathcal{F}^{t-1}] ( \langle \E[g_i^{t} | \mathcal{F}^{t-1}] , \overline{\theta}^*\rangle -c_i) + \sum_{\tau=1}^{t-1}\sum_{i=1}^IQ_i(\tau)(\langle g_i^{\tau}, \overline{\theta}^*\rangle-c_i)  \\
\leq & \sum_{\tau=1}^{t-1}\sum_{i=1}^IQ_i(\tau)(\langle g_i^{\tau}, \overline{\theta}^*\rangle-c_i) = \E [Z(t-1)],
\end{align*}
where $Q_i(t)$ and $g_i^{t}$ are independent variables with $Q_i(t) \geq 0$ and $\langle \E[g_i^{t} | \mathcal{F}^{t-1}], \overline{\theta}^*\rangle \leq c_i$. On the other hand, we can know the random process has bounded drifts as
\begin{align*}
|Z(t+1) - Z(t)| = & \sum_{i=1}^I Q_i(t+1) (\langle g_i^{t}, \overline{\theta}^* \rangle - c_i) \\
\leq & \|\mathbf{Q}(t+1)\|_2 \sqrt{\sum_{i=1}^I \big|\langle g_i^{t+1}, \overline{\theta}^* \rangle - c_i \big |^2 } \\
\leq & \|\mathbf{Q}(t+1)\|_2 \sum_{i=1}^I (\| g_i^{t+1}\|_\infty \|\overline{\theta}^* \|_1 + |c_i|) \\
\leq &  2 L\|\mathbf{Q}(t+1)\|_2,
\end{align*}
where the first inequality is by Cauchy-Schwarz inequality, and the last inequality is by Assumption \ref{as:bounded}. This also implies that for an arbitrary $\varsigma$, we have $\{ |Z(t+1) - Z(t) | > \varsigma  \} \subseteq \{ Y(t):=\|\mathbf{Q}(t+1)\|_2 - \varsigma /(2L) > 0 \}$ since $|Z(t+1) - Z(t) | > \varsigma$ implies $2 L \|\mathbf{Q}(t+1)\|_2 > \varsigma $ according to the above inequality. Thus, by Lemma \ref{lem:supermar}, we have
\begin{align}
\begin{aligned}\label{eq:Q_prod_prob}
\Pr\bigg(\sum_{t=1}^T\sum_{i=1}^IQ_i(t)(\langle g_i^{t}, \overline{\theta}^*\rangle-c_i) \geq z \bigg) &\leq e^{-z^2/(2T\varsigma^2)} + \sum_{t = 0}^{T-1} \Pr \bigg( \|\mathbf{Q}(t+1)\|_2 > \frac{\varsigma}{2 L} \bigg )\\
&=e^{-z^2/(2T\varsigma^2)} + \sum_{t = 1}^{T} \Pr \bigg( \|\mathbf{Q}(t)\|_2 > \frac{\varsigma}{2 L} \bigg ), 
\end{aligned}
\end{align}
where we can see that bounding $\|\mathbf{Q}(t)\|_2$ is the key to obtaining the bound of $\sum_{t=1}^T \sum_{i=1}^I Q_i(t) (\langle g_i^{t}, \overline{\theta}^* \rangle -c_i)$.

Next, we will show the upper bound of the term $\|\mathbf{Q}(t)\|_2$. According to the proof of Lemma \ref{lem:Q_drift_diff}, if $\overline{\sigma}/4 \geq \zeta  (\overline{\sigma}/2 + 2 L )$, setting
\begin{align*}
&\psi=\frac{  2 \tau L + C_{V,\alpha, \lambda}}{\overline{\sigma}} + \frac{2\alpha  L \log ( |\mathcal{S}|^2 |\mathcal A|/ \lambda) }{\overline{\sigma}\tau } + \frac{\tau  \overline{\sigma}}{2}, \\
&C_{V, \alpha, \lambda} := 2V \bigg( \overline{\sigma} B + 3L  + 2 \overline{\vartheta}L +  \overline{\sigma} \overline{\vartheta}  +  \frac{L }{2\alpha} + 2L\lambda  + \frac{\alpha \lambda L \log |\mathcal S|^2 |\mathcal A|+ 4 L^2}{V}\bigg),
\end{align*}
we have that with probability at least $1-\delta$, for a certain $t \in [T+1-\tau]$, the following inequality holds
\begin{align*}
\|\mathbf{Q}(t)\|_2 \leq \psi + \tau  \frac{512 L^2 }{ \overline{\sigma}} \log[1+\frac{128 L^2 }{\overline{\sigma}^2} e^{\overline{\sigma} / (32 L )}] + \tau  \frac{64 L^2 }{\overline{\sigma}} \log \frac{1}{\delta}+2\tau L.
\end{align*}
Thus, combining \eqref{eq:Q_shift} and the above inequality at $t=T+1-\tau$, with probability at least $1-\delta$, for a certain $t$ satisfying $ T+2-\tau\leq t \leq T+1$, the above inequality also holds. The above inequality is equivalent to 
\begin{align*}
\Pr \l( \|\mathbf{Q}(t)\|_2 > \psi + \tau  \frac{512 L^2 }{ \overline{\sigma}} \log[1+\frac{128 L^2 }{\overline{\sigma}^2} e^{\overline{\sigma} / (32 L)}] + \tau  \frac{64 L^2 }{\overline{\sigma}} \log \frac{1}{\delta} + 2\tau L \r) \leq \delta.
\end{align*}
Setting $\varsigma = 2 L \psi + \tau  \frac{1024 L^3}{ \overline{\sigma}} \log \big [1+\frac{128 L^2 }{\overline{\sigma}^2} e^{\overline{\sigma} / (32 L)}\big] + \tau  \frac{128 L^3 }{\overline{\sigma}} \log \frac{1}{\delta} + 4\tau L^2$ and $z = \sqrt{2T\varsigma^2 \log\frac{1}{T\delta}}$ in \eqref{eq:Q_prod_prob}, then the following probability hold with probability at least $1 - 2T\delta$ with
\begin{align*}
&\sum_{t=1}^{T}\sum_{i=1}^IQ_i(t)(\langle g_i^t, \overline{\theta}^*\rangle-c_i) \\
& \qquad \leq \bigg(2L \psi + \tau  \frac{1024 L^3 }{ \overline{\sigma}} \log\Big[1+\frac{128 L^2 }{\overline{\sigma}^2} e^{\overline{\sigma} / (32 L )}\Big] + \tau  \frac{128 L^3 }{\overline{\sigma}} \log \frac{1}{\delta} + 4\tau L^2 \bigg) \sqrt{T\log \frac{1}{T\delta}}, 
\end{align*}
which completes the proof.
\end{proof}

\vspace{-0.35cm}
\section{Proofs of the Lemmas in  Section \ref{sec:constraint}} \label{sec:B}
\vspace{-0.15cm}
\subsection{Proof of Lemma \ref{lem:term_III_all}}


\begin{proof}[Proof of Lemma \ref{lem:term_III_all}] We start our proof with the updating rule of $\mathbf{Q}(\cdot)$ as follows
\begin{align*}
Q_i(t) =& \max \{ Q_i(t-1) + \langle g_i^{t-1} , \theta^t \rangle -c_i, 0  \}\geq  Q_i(t-1) + \langle g_i^{t-1} , \theta^t \rangle -c_i \\
\geq & Q_i(t-1) + \langle g_i^{t-1} , \theta^{t-1} \rangle -c_i +  \langle g_i^{t-1} , \theta^t  - \theta^{t-1} \rangle. 
\end{align*}
Rearranging the terms in the above inequality further leads to
\begin{align*}
\langle g_i^{t-1} , \theta^{t-1} \rangle -c_i  \leq & Q_i(t) - Q_i(t-1)- \langle g_i^{t-1}, \theta^t  - \theta^{t-1} \rangle. 
\end{align*}
Thus, taking summation on both sides of the above inequality from $2$ to $T+1$ leads to 
\begin{align*}
&\sum_{t=1}^{T} \l( \langle g_i^{t} , \theta^{t} \rangle -c_i \r) \leq  Q_i(T+1) - \sum_{t=1}^{T} \langle g_i^{t}, \theta^{t+1}  - \theta^{t} \rangle  \leq Q_i(T+1) + \sum_{t=1}^{T}  \| g_i^{t} \|_{\infty}  \|\theta^{t+1}  - \theta^{t}\|_1,
\end{align*}
where the second inequality is due to H\"older's inequality. Note that the right-hand side of the above inequality is no less than $0$ since $Q_i(t) = \max \{ Q_i(t-1) + \langle g_i^{t-1} , \theta^t \rangle -c_i, 0 \} \geq 0$. Thus, we have 
\begin{align*}
\Bigg[ \sum_{t=1}^{T} \l( \langle g_i^{t} , \theta^{t} \rangle -c_i \r) \Bigg]_{+} \leq 
Q_i(T+1) +  \sum_{t=1}^{T}  \| g_i^{t} \|_{\infty}  \|\theta^{t+1}  - \theta^{t}\|_1,
\end{align*}
where $[~\cdot~]_+$ is an entry-wise application of the operation $\max\{\cdot, 0\}$ for any vector.

Defining $\mf g^t(\theta^t) := [\langle g_1^{t} , \theta^{t} \rangle, \cdots,  \langle g_I^{t} , \theta^{t} \rangle]^\top $ and $ \mf c := [c_1, \cdots,c_I]^\top$, we obtain
\begin{align*}
\Bigg\| \Bigg[ \sum_{t=1}^{T} \l( \mf g^t(\theta^t) - \mf c \r) \Bigg]_{+} \Bigg\|_2  \leq &
\|\mathbf{Q}(T+1)\|_2 +  \sum_{t=1}^{T}  \sqrt{\sum_{i=1}^I \| g_i^{t} \|^2_{\infty}}  \|\theta^{t+1}  - \theta^{t}\|_1 \\  
\leq & \|\mathbf{Q}(T+1)\|_2 +  \sum_{t=1}^{T}  \sum_{i=1}^I \| g_i^{t} \|_{\infty}  \|\theta^{t+1}  - \theta^{t}\|_1 \\
\leq & \|\mathbf{Q}(T+1)\|_2 + \sum_{t=1}^{T}    \|\theta^{t+1}  - \theta^{t}\|_1, 
\end{align*}
where the third inequality is due to Assumption \ref{as:bounded}. This completes the proof.
\end{proof}

\subsection{Proof of Lemma \ref{lem:term_III_th_diff}}

\begin{lemma} [Proposition 18 of \citet{jaksch2010near}] \label{lem:epoch_bound} The number of epochs up to episode $T$ with $T \geq  |\mathcal S| |\mathcal A|$ is upper bounded by 
\begin{align*}
    \ell(T) \leq |\mathcal S| |\mathcal A| \log_2 \l(\frac{8T}{|\mathcal S| |\mathcal A|} \r) \leq \sqrt{T|\mathcal S| |\mathcal A|} \log_2 \l(\frac{8T}{|\mathcal S| |\mathcal A|} \r),
\end{align*}
where $\ell(\cdot)$ is a mapping from a certain episode to the epoch where it lives.
\end{lemma}

We are ready to give the proof of Lemma \ref{lem:term_III_th_diff}.

\begin{proof}[Proof of Lemma \ref{lem:term_III_th_diff}] We need to discuss the upper bound of the term $\|\theta^{t+1} - \theta^{t}\|_1$ in the following two cases: 
\begin{itemize}
\item[\textbf{(1)}] $\ell(t+1) = \ell(t)$, i.e., episodes $t+1$ and $t$ are in the same epoch;
\item[\textbf{(2)}] $\ell(t+1) > \ell(t)$, i.e., episodes $t+1$ and $t$ are in two different epochs.
\end{itemize}

For the first case where $\ell(t+1) = \ell(t)$, according to Lemma \ref{lem:breg_tria}, letting $\mathbf{x}^{opt} = \theta^t$, $\mathbf{y} = \widetilde{\theta}^{t-1}$, $\mathbf{z} =\theta^{t-1}$  and $h(\theta) = \big\langle Vf^{t-1} + \sum_{i=1}^IQ_i(t-1)g_i^{t-1}, \theta \big\rangle$, with $t\geq 2$ and $\ell(t) = \ell(t-1)$, we have
\begin{align*}
&\Big\langle Vf^{t-1} + \sum_{i=1}^IQ_i(t-1)g_i^{t-1}, \theta^t \Big\rangle + \alpha D(\theta^t, \widetilde{\theta}^{t-1}) \\
&\qquad \leq \Big\langle Vf^{t-1} + \sum_{i=1}^IQ_i(t-1)g_i^{t-1} ,  \widetilde{\theta}^{t-1} \Big\rangle + \alpha D( \theta^{t-1},\widetilde{\theta}^{t-1} ) - \alpha D(\theta^{t-1}, \theta^t ).
\end{align*}
Rearranging the terms and dropping the last term (due to $D(\theta^{t-1}, \theta^t ) \geq 0$) yield 
\begin{align*}
\alpha D(\theta^t, \widetilde{\theta}^{t-1}) \leq & \Big \langle Vf^{t-1} + \sum_{i=1}^IQ_i(t-1)g_i^{t-1} , 
\theta^{t-1} - \theta^t\Big \rangle + \alpha D( \theta^{t-1},\widetilde{\theta}^{t-1} ) \\
\leq & \Big( V\|f^{t-1}\|_\infty + \sum_{i=1}^IQ_i(t-1) \|g_i^{t-1} \|_\infty \Big) \|  \theta^{t-1} - \theta^t\|_1 + \alpha D( \theta^{t-1},\widetilde{\theta}^{t-1} )\\
\leq & \l( V + \|\mathbf{Q}(t-1)\|_2 \sqrt{\sum_{i=1}^I\|g_i^{t-1} \|^2_\infty} \r) \|  \theta^{t-1} - \theta^t\|_1 + \alpha D( \theta^{t-1},\widetilde{\theta}^{t-1} )\\
\leq & ( V + \|\mathbf{Q}(t-1)\|_2  ) \| \theta^{t-1} - \theta^t\|_1 + \alpha \lambda L \log |\mathcal{S}|^2 |\mathcal{A}| ,
\end{align*}
where the second inequality is by H\"older's inequality and triangle inequality, the third inequality is by Cauchy–Schwarz inequality and Assumption \ref{as:bounded}, and the last inequality is due to Assumption \ref{as:bounded} and the first inequality in Lemma \ref{lem:breg_diff} with setting $\theta = \theta' = \theta^{t-1}$ and $\widetilde{\theta}' = \tilde{\theta}^{t-1}$. Note that by Lemma \ref{lem:pins}, there is
\begin{align*}
D(\theta^t, \widetilde{\theta}^{t-1}) \geq \frac{1}{2L}\|\theta^t - \widetilde{\theta}^{t-1}\|_1^2.
\end{align*}
Thus, combining the previous two inequalities, we obtain
\begin{align*}
\|\theta^t - \widetilde{\theta}^{t-1}\|_1^2 \leq \frac{2L V + 2L \|\mathbf{Q}(t-1)\|_2 }{\alpha} \|\theta^t - \theta^{t-1}\|_1 + 2 \lambda L^2 \log |\mathcal{S}|^2 |\mathcal{A}|,
\end{align*}
which further leads to
\begin{align*}
\|\theta^t - \widetilde{\theta}^{t-1}\|_1 \leq  \sqrt{\frac{ 2LV + 2L\|\mathbf{Q}(t-1)\|_2 }{\alpha} \|\theta^{t-1} - \theta^t\|_1} +   \sqrt{ 2\lambda L^2 \log |\mathcal{S}|^2 |\mathcal{A}|}.
\end{align*}
Since there is 
\begin{align*}
\|\theta^t - \widetilde{\theta}^{t-1}\|_1 &=  \sum_{k=0}^{L-1} \l \|\theta_k^t - (1-\lambda)\theta_k^{t-1} - \lambda \frac{1}{|\mathcal{S}|^2 |\mathcal{A}|} \mf 1\r\|_1 \\
&\geq (1-\lambda)\|\theta^t - \theta^{t-1}\|_1 -  \lambda L,
\end{align*} 
where $\theta_k:=[\theta(s, a, s')]_{s \in \mathcal{S}_k, a \in \mathcal{A}, s' \in \mathcal{S}_{k+1}}$, combining it with the last inequality, we further have
\begin{align*}
\|\theta^t - \theta^{t-1}\|_1 &\leq  \sqrt{\frac{ 2LV + 2L\|\mathbf{Q}(t-1)\|_2 }{\alpha(1-\lambda)^2} \|\theta^{t-1} - \theta^t\|_1} +   \frac{\sqrt{ 2\lambda L^2 \log |\mathcal{S}|^2 |\mathcal{A}|}}{(1-\lambda)} + \frac{\lambda L}{1-\lambda}\\
&\leq  \frac{ 2LV + 2L\|\mathbf{Q}(t-1)\|_2 }{2(1-\lambda)^2\alpha}  +  \frac{1}{2}\|\theta^{t-1} - \theta^t\|_1 +  \frac{\sqrt{ 2\lambda L^2 \log |\mathcal{S}|^2 |\mathcal{A}|}}{1-\lambda} + \frac{\lambda L}{1-\lambda},
\end{align*}
where the last inequality is due to $\sqrt{ab}\leq |a|/2 + |b|/2$. Rearranging the terms in the above inequality gives for $t\geq 2$ with $\ell(t) = \ell(t-1)$,
\begin{align*}
\|\theta^t - \theta^{t-1}\|_1 &\leq \frac{ 2LV + 2L\|\mathbf{Q}(t-1)\|_2 }{(1-\lambda)^2\alpha} +  \frac{\sqrt{8\lambda L^2 \log |\mathcal{S}|^2 |\mathcal{A}|}}{1-\lambda} + \frac{2\lambda L}{1-\lambda}.
\end{align*}
Shifting the index in the above inequality, we further have for $t\in[T]$ with $\ell(t+1) = \ell(t)$,
\begin{align}
\|\theta^{t+1} - \theta^t\|_1 &\leq \frac{ 2LV + 2L\|\mathbf{Q}(t)\|_2 }{(1-\lambda)^2\alpha} +  \frac{\sqrt{8\lambda L^2 \log |\mathcal{S}|^2 |\mathcal{A}|}}{1-\lambda} + \frac{2\lambda L}{1-\lambda}. \label{eq:cons_viol_theta_case1}
\end{align}

For the second case where $\ell(t+1) > \ell(t)$ with $t\in [T]$, it is difficult to know whether the two solutions $\theta^{t+1}$ and $\theta^t$ are in the same feasible set since $\Delta(\ell(t+1), \zeta) \neq \Delta(\ell(t), \zeta)$. Thus, the above derivation does not hold. Then, we give a bound for the term $\|\theta^{t+1} - \theta^{t}\|_1$ as follows
\begin{align}
\begin{aligned}\label{eq:cons_viol_theta_case2}
\|\theta^{t+1} - \theta^{t}\|_1 &\leq   \|\theta^{t+1}\|_1 + \| \theta^{t}\|_1 \\
&=  \sum_{k=0}^{L-1} \sum_{s\in \mathcal{S}_k} \sum_{a\in \mathcal{A}} \sum_{s'\in \mathcal{S}_{k+1}}[\theta^t(s, a, s') +\theta^t(s, a, s')]  = 2L.
\end{aligned}
\end{align}
However, we can observe that $\ell(t+1) > \ell(t)$ only happens when $t+1$ is a starting episode for a new epoch. The number of starting episodes for new epochs in $T+1$ episodes is bounded by $\ell(T)$, namely the total number of epochs in $T$ episodes. According to Lemma \ref{lem:epoch_bound}, the total number of epochs $\ell(T)$ is bounded as $\ell(T) \leq  \sqrt{T|\mathcal{S} | |\mathcal{A} |} \log_2 [8T / ( |\mathcal{S} | |\mathcal{A} | )] \leq 1.5 \sqrt{T|\mathcal{S} | |\mathcal{A} |} \log_2 [8T / ( |\mathcal{S} | |\mathcal{A} | )]$ which only grows in the order of $\sqrt{T}\log T$. 

Thus, we can decompose the term $\sum_{t=1}^T \|\theta^{t+1}-\theta^{t}\|_1$ in the following way
\begin{align*}
\sum_{t=1}^T \|\theta^{t+1}-\theta^{t}\|_1 =&  \sum_{\substack{t:~t\leq T, \\ \ell(t+1) > \ell(t)} } \|\theta^{t+1}-\theta^{t}\|_1 + \sum_{\substack{t:~t\leq T, \\ \ell(t+1) = \ell(t)}} \|\theta^{t+1}-\theta^{t}\|_1 \\
\leq&  2L \ell(T) + \sum_{\substack{t:~t\leq T, \\ \ell(t+1) = \ell(t)}} \|\theta^{t+1}-\theta^{t}\|_1,
\end{align*}
where the inequality is due to \eqref{eq:cons_viol_theta_case2} and  the fact that $\sum_{\substack{t:~t\leq T, \\ \ell(t+1) > \ell(t)}} 1 \leq \ell(T)$. By \eqref{eq:cons_viol_theta_case1}, we can further bound the last term in the above inequality as
\begin{align*}
\sum_{\substack{t:~t\leq T, \\ \ell(t+1) = \ell(t)}} \|\theta^{t+1}-\theta^t\|_1 & \leq \frac{ 2TLV + 2L\sum_{t=1}^{T}\|\mathbf{Q}(t)\|_2 }{(1-\lambda)^2\alpha} +  \frac{\sqrt{8\lambda \log |\mathcal{S}|^2 |\mathcal{A}|}}{1-\lambda}TL + \frac{2\lambda}{1-\lambda} TL,
\end{align*}
where we relax the summation on the right-hand side to $\sum_{t=1}^T$. Thus, we eventually obtain
\begin{align*}
\sum_{t=1}^T \|\theta^{t+1}-\theta^t\|_1 & \leq 2L \ell(T) + \sum_{\substack{t:~t\leq T, \\ \ell(t+1) = \ell(t)}} \|\theta^{t+1}-\theta^t\|_1 \\
& \leq 3L  \sqrt{T|\mathcal{S} | |\mathcal{A} |} \log  \frac{8T}{ |\mathcal{S} | |\mathcal{A} | }  + \frac{2L }{(1-\lambda)^2\alpha} \sum_{t=1}^{T}\|\mathbf{Q}(t)\|_2 + \frac{ 2VLT }{(1-\lambda)^2\alpha} \\
&\quad + \frac{2\lambda LT}{1-\lambda} +  \frac{\sqrt{8\lambda \log |\mathcal{S}|^2 |\mathcal{A}|}}{1-\lambda}LT,
\end{align*}
where we use the result in Lemma \ref{lem:epoch_bound} to bound the number of epoch, i.e., $\ell(T)$. This completes the proof.
\end{proof}